\documentclass{article}

\PassOptionsToPackage{round,authoryear}{natbib}

\usepackage[preprint]{neurips_2026}

\usepackage[utf8]{inputenc} 
\usepackage[T1]{fontenc}    
\usepackage{hyperref}       
\hypersetup{hypertexnames=false} 
\usepackage{url}            
\usepackage{booktabs}       
\usepackage{multirow}
\usepackage{graphicx}
\usepackage{amsfonts}       
\usepackage{nicefrac}       
\usepackage{microtype}      
\usepackage[table]{xcolor} 
\usepackage{algorithm}
\usepackage{algpseudocode}
\usepackage{amsmath}
\usepackage{amsthm, amssymb}
\usepackage{todonotes}
\usepackage{changes}

\usepackage{wrapfig}

\newtheorem{theorem}{Theorem}[section]
\newtheorem{lemma}[theorem]{Lemma}
\newtheorem{proposition}{Proposition}

\title{Accelerating Zeroth-Order Spectral Optimization with Partial Orthogonalization from Power Iteration}

\author{%
  Jiahe Chen, Ziye Ma \\
  Department of Computer Science\\
  City University of Hong Kong\\
  \texttt{cjiahe2-c@my.cityu.edu.hk, ziyema@cityu.edu.hk} \\
}

\begin{document}

\maketitle

\begin{abstract}
  Zeroth-order (ZO) optimization has become increasingly popular and important in fine-tuning large language models (LLMs), especially on edge devices due to its ability to adjust the model to local data without the need for memory-intensive back-propagation. Recent works try to reduce ZO variance through low-dimensional subspace search, but subspace restriction alone leaves key optimization geometry under-exploited, motivating additional acceleration. In this work, we focus on the hidden layer training problem in which spectral optimizers like Muon outperform AdamW due to its ability to exploit weak spectral directions by orthogonalization. However, we have discovered that unlike in the first-order setting, full orthogonalization works poorly in the ZO setting since the gradient estimates are highly noisy and unreliable. To address this issue, we propose applying partial spectral orthogonalization to accelerate ZO optimization. To do so, we replace the iconic Newton-Schulz procedure in Muon with the faster, more concentrated power-iteration method so that it only amplifies dominant spectral directions. Furthermore, to improve the efficiency and generalization of the algorithm, we adopted a streaming variant of power-iteration that requires low variance in gradients, which was achieved through constraining our search inside a subspace obtained through the projection of momentum, echoing recent advances. Experiments on LLM fine-tuning show that our method can achieve from 1.5x to 4x the convergence speed of ZO-Muon, the current SOTA algorithm, across SuperGlue datasets in the OPT-13B model. Across different models, we also reach competitive final accuracies with less time in most cases compared with strong ZO baselines such as MeZO, LOZO and ZO-Muon. Code is
available at https://github.com/MOFA-LAB/ZO-MOPI.git.
\end{abstract}

\section{Introduction}
\label{sec:intro}
Large language models (LLMs) have achieved strong performance across a wide range of tasks. Fine-tuning is a standard way to adapt pretrained LLMs to downstream applications~\citep{gururangan2020don,sanh2021multitask}. Most existing fine-tuning pipelines rely on first-order (FO) optimization methods such as Adam~\citep{kingma2014adam} or its variants, but given that modern models are so huge, the dynamic memory required~\citep{malladi2023fine,zhang2024revisiting} to perform even small batch back-propagation is out of reach for most edge devices like autonomous vehicles and smartphones, and are usually reserved for data centers. To circumvent this issue, more practitioners are evaluating the possibility to perform zeroth-order (ZO) optimization instead, since it only relies on forward passes, enabling memory-constrained devices to fine-tune their models locally to new data. Despite recent progress, ZO optimizers still remain to be significantly slower than their FO counterparts, demanding more attention and in-depth studies.

ZO methods that have been studied over the past few decades were mainly studied for black-box optimization problems where explicit gradient information is unavailable or inaccessible~\citep{liu2020primer,chen2017zoo,ilyas2018black}. Recently, \cite{malladi2023fine} showed that ZO methods can also be effective for LLM fine-tuning through their famous MeZO algorithm. However, ZO methods estimate gradients through Monte Carlo perturbations rather than direct backpropagation, and the variance of these estimates grows with the number of parameters~\citep{nesterov2017random,duchi2015optimal}. In the LLM regime, where models often contain billions of parameters, vanilla ZO gradient estimates can therefore become highly noisy without additional variance reduction. This challenge is especially pronounced when applying spectral optimizers such as Muon~\citep{jordan2024muon} to hidden-layer weights. Unlike Adam-style methods, Muon does not rely on second-moment statistics; instead, it uses orthogonalization to exploit the eigenspaces of matrix-valued gradients. As a result, its effectiveness depends critically on the quality of the momentum and gradients used in the algorithm. Figure~\ref{fig:motivation_singular_values} illustrates this issue: even at the same point, the spectrum of the ZO-estimated gradient differs substantially from that of the true gradient. This mismatch helps explain why direct applications of Muon in the ZO setting, such as~\citep{petrov2025leveraging}, yield limited gains. In particular, Muon's Newton–Schulz orthogonalization equalizes energy across spectral directions; when applied to noisy ZO estimates, it can inadvertently amplify inaccurate directions, turning estimation error into harmful update components.
\begin{figure}[t]
    \centering
    \begin{tabular}{cc}
        \includegraphics[width=0.4\linewidth]{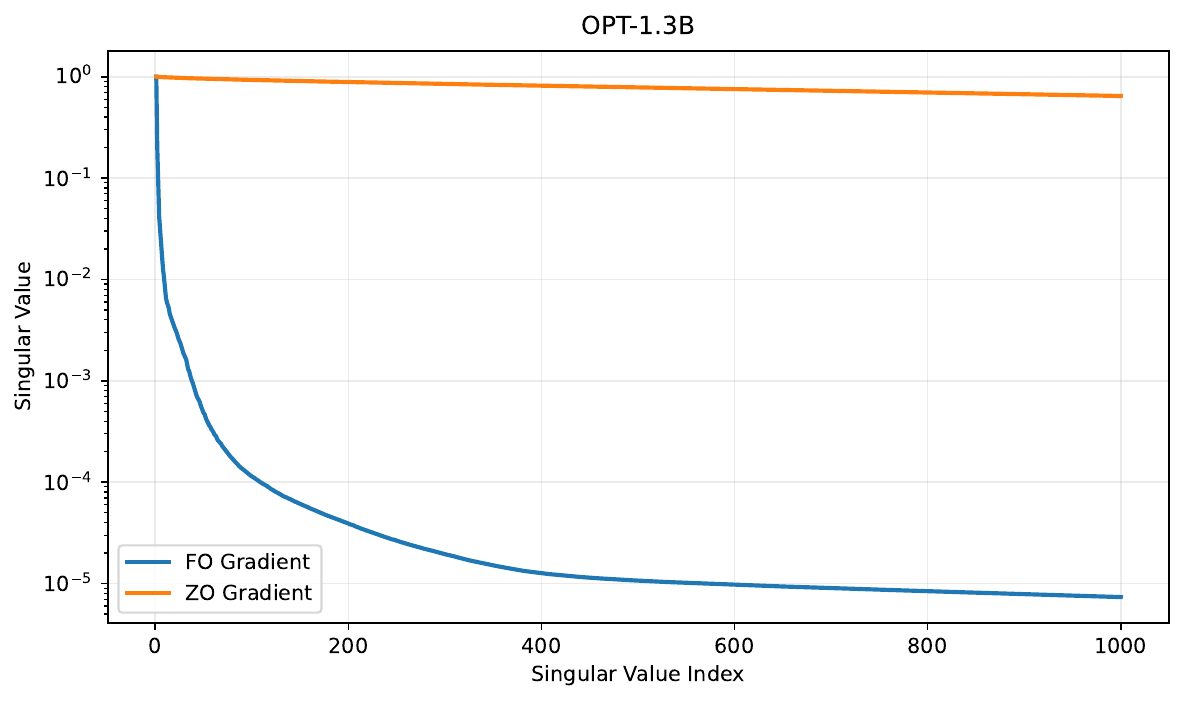} &
        \includegraphics[width=0.4\linewidth]{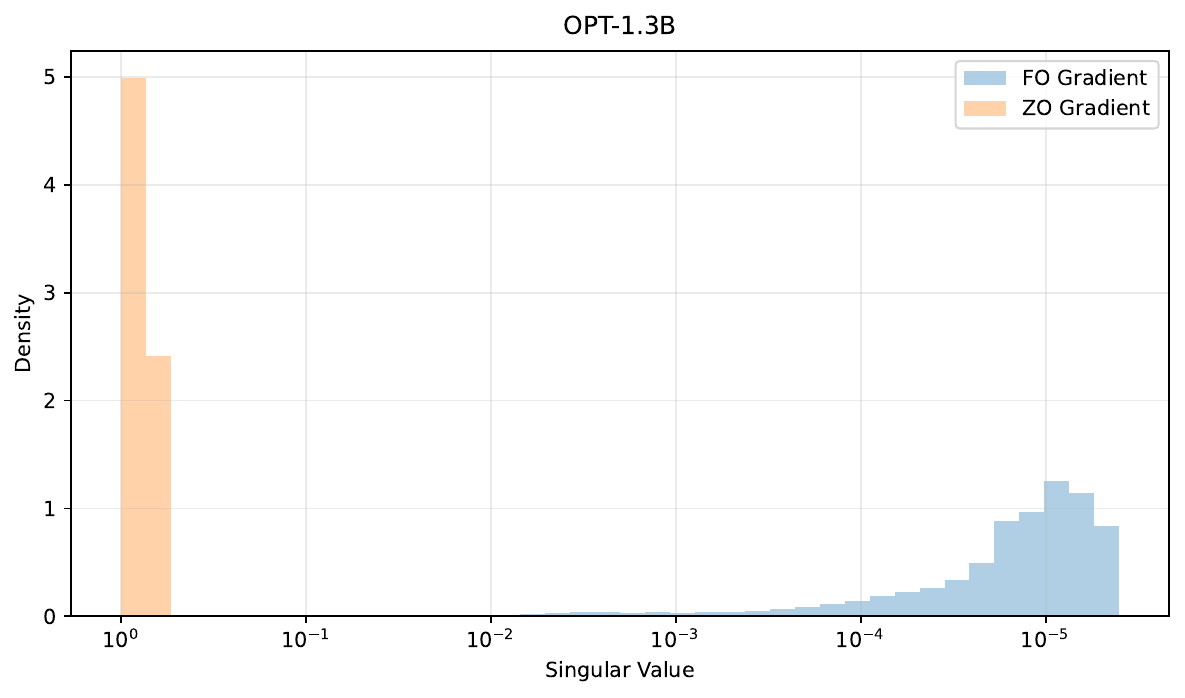} \\
        \small (a) Singular-value spectrum &
        \small (b) Singular-value density
    \end{tabular}
    \caption{Comparison of FO and ZO gradient spectra on the \texttt{fc2} weight of the first decoder block in OPT-1.3B, computed on the same mini-batch at the same parameter point.  FO gradients are concentrated in dominant directions, while ZO estimates contain a heavier and flatter spectral tail.}
    \label{fig:motivation_singular_values}
\end{figure}
These observations motivate a more selective spectral alignment for ZO optimization. Instead of orthogonalizing the entire active spectrum like Newton--Schulz in traditional Muon, we aim to preserve the reliable spectral components while avoiding amplification of noise-dominated directions. This is also supported by recent observations on spectral anisotropy in LLM fine-tuning~\citep{huang2026spectra}: different singular directions in low-rank gradients have different robustness to noise. Therefore we adopt partial orthogonalization, which only orthogonalizes a small number of dominant singular directions while suppressing most of the remaining noisy directions.

We implement partial orthogonalization using a technique known as Streaming Power Iteration (SPI)~\citep{huang2026spectra,kexuefm-11654}, a classic online algorithm that aims to efficiently extract singular directions of a matrix. To cater to our needs, we further adopted a version of SPI that only computes the top-$k$ singular spaces, while discarding the other nosier, less reliable singular directions that are prevalent in ZO optimization. Furthermore, to improve the efficiency and generalization of the algorithm, we constrain our ZO sampling only to a subspace and maintain a projected momentum within the subspace. The subspace estimator reduces the effective search dimension and provides gradients with lower variance for spectral tracking. The projected momentum extracts the most important directional information in the past gradient estimates with negligible memory overhead, and provides the stability needed by SPI, since it makes the dominant gradient directions evolve more smoothly across iterations. As an early illustration of our method’s potential, Figure~\ref{fig:intro_gemma} shows ZO fine-tuning results for Gemma2-2B on several SuperGLUE tasks~\citep{wang2019superglue,team2024gemma}. Partial orthogonalization consistently accelerates training with negligible additional overhead, suggesting that spectral structure can be better exploited even under noisy ZO estimates. We will provide the full methodology and a more comprehensive empirical evaluation in later sections. 

\begin{figure}[t]
    \centering
    \setlength{\tabcolsep}{1pt}
    \begin{tabular}{@{}ccc@{}}
        \includegraphics[height=0.19\linewidth]{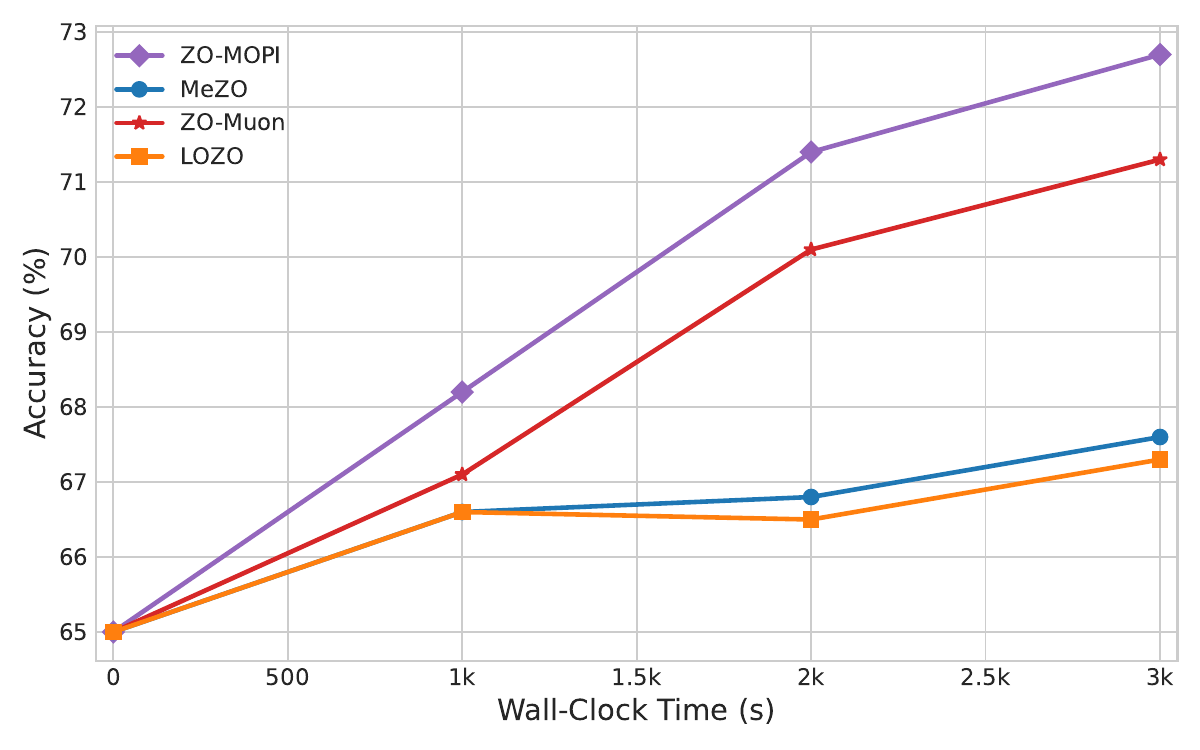} &
        \includegraphics[height=0.19\linewidth]{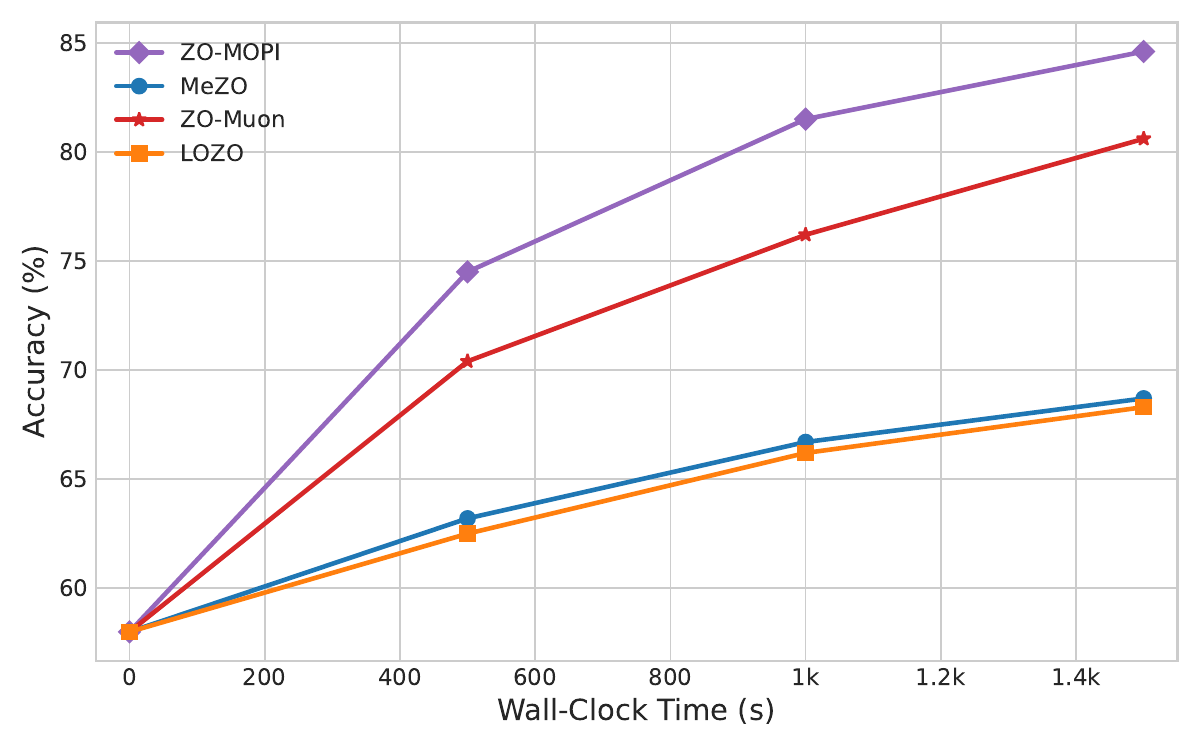} &
        \makebox[0.37\linewidth][c]{\includegraphics[height=0.21\linewidth]{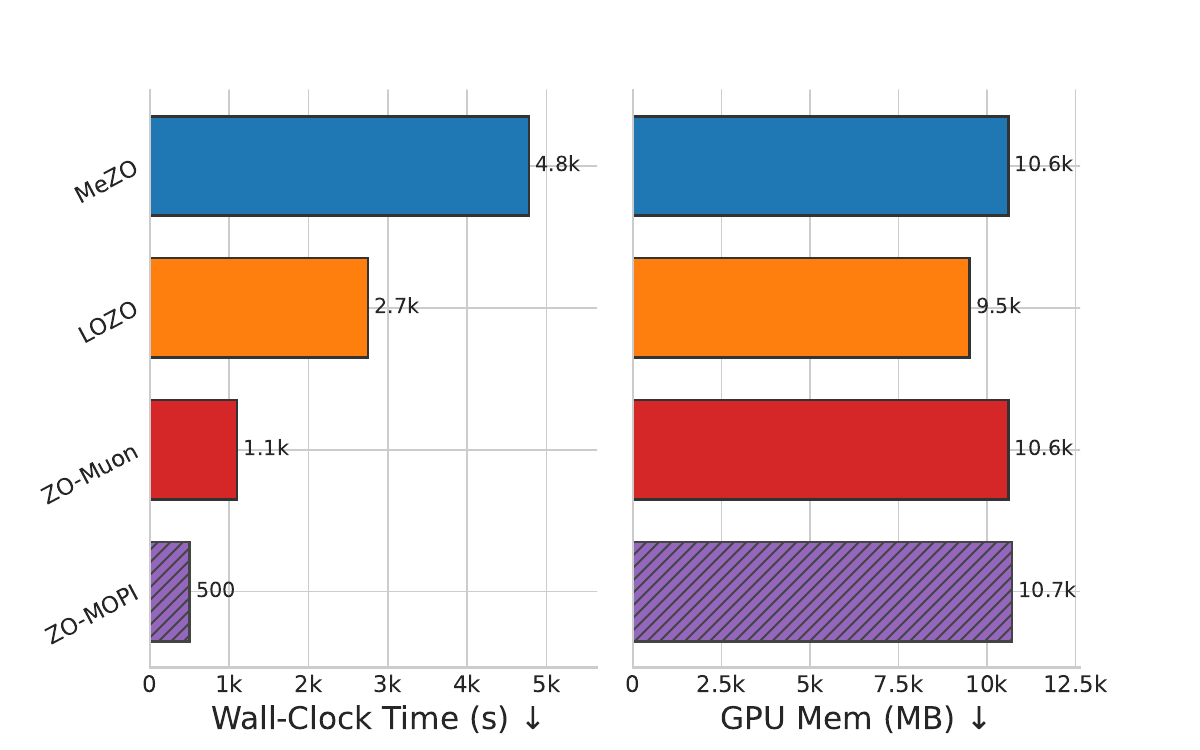}} \\
        \multicolumn{2}{c}{\makebox[0.61\linewidth][c]{\footnotesize (a) LLM fine-tuning: Gemma2-2B on BoolQ/SQuAD}} &
        \makebox[0.37\linewidth][c]{\footnotesize (b) Efficiency comparison: Gemma2-2B on SST-2}
    \end{tabular}
    \caption{Gemma2-2B fine-tuning comparison. Panel (a) reports test accuracy versus wall-clock time on BoolQ and SQuAD. Panel (b) compares wall-clock time and peak memory on SST-2, achieving 92\% test accuracy. Our method reaches strong accuracy faster while preserving the memory efficiency of ZO fine-tuning.}
    \label{fig:intro_gemma}
\end{figure}

\section{Related Work}
\label{sec:related}

\paragraph{Zeroth-order optimization in deep model training.}
Zeroth-order (ZO) optimization estimates gradients via finite differences of function evaluations, without requiring explicit FO gradient computations~\citep{ghadimi2013stochastic,duchi2015optimal,nesterov2017random,liu2020primer}. Early ZO studies in deep learning mainly focused on black-box optimization such as adversarial attack and defense in small-scale black-box models, where gradients are unavailable or expensive to compute~\citep{chen2017zoo,ilyas2018black,cheng2019improving,tu2019autozoom}. More recently, \citep{malladi2023fine} applied ZO-SGD~\citep{ghadimi2013stochastic} in LLM fine-tuning, which proves that ZO methods can also be effective in training of large models in constrained environments. Since then, the application of ZO methods has increasingly focused on memory-efficient training in LLMs~\citep{zhang2024revisiting,liu2024sparse,chen2024enhancing,zhao2024second,yu2025zeroth}.

\paragraph{Variance reduction methods in ZO optimization.}
A central challenge in ZO optimization is the high variance of gradient estimates, which becomes especially severe for large-scale models. Prior work has reduced this variance through control-variate methods, adaptive momentum, second-order preconditioning, and low-dimensional subspace estimation~\citep{gautam2024variance,zhao2024second,chen2024enhancing,shu2025refining,yu2025zeroth,lang2026powering}. Among them, most methods will cause large additional memory or runtime overhead, limiting their application in resource-constrained settings. Subspace ZO methods are particularly attractive for resource-constrained training because they exploit the low-rank structure commonly observed in LLM gradients. LOZO~\citep{chen2024enhancing} derives the ZO gradient using a low-rank perturbation matrix, ensuring that the approximated gradient retains a low-rank structure. SubZero~\citep{yu2025zeroth} generates a low-rank perturbation matrix by combining two column-orthogonal matrices with a Gaussian random matrix, which is then used for gradient estimation. By restricting random perturbations to a low-dimensional subspace, they reduce the effective search dimension and thus lower the variance of ZO gradient estimates with little additional memory overhead. Both LOZO and SubZero only apply subspace ZO estimation within the traditional SGD framework. Beyond this, subspace ZO can be combined with more sophisticated optimizers, as demonstrated by ZO-Muon~\citep{lang2026powering}, which integrates it with Muon's gradient orthogonalization~\citep{jordan2024muon} for improving the trade-off between accuracy and query efficiency.

\paragraph{Spectral optimization methods.}
Traditional optimizers such as Adam~\citep{kingma2014adam} are vector-based and operate on neural network weights in an element-wise manner, which may under-utilize matrix structure. To address this limitation, a growing line of work explores structure-aware optimization strategies that leverage matrix geometry and spectral information. Eigenspectrum-aware preconditioned optimizers such as Shampoo~\citep{gupta2018shampoo} and SOAP~\citep{vyas2024soap} rescale gradients using matrix inverse-root preconditioners or updates in the preconditioner's eigenbasis, leading to improved convergence in large-scale training. More recently, Muon~\citep{jordan2024muon} has gained popularity as a lightweight alternative, using Newton--Schulz (NS) iterations to perform gradient orthogonalization with minimal overhead, achieving strong performance in FO training~\citep{liu2025muon}. However, directly adapting Muon to ZO optimization has shown limited gains over strong ZO baselines in LLM fine-tuning~\citep{petrov2025leveraging}.


\section{Preliminaries}
\label{sec:preliminaries}

\paragraph{Notations}
We use bold uppercase letters to denote matrix-valued parameters and perturbations. For a matrix $\mathbf{X}\in\mathbb{R}^{m\times n}$, $\mathbf{X}^\top$ denotes its transpose, $\|\mathbf{X}\|_F$ denotes its Frobenius norm, and $\operatorname{rank}(\mathbf{X})$ denotes its rank. Let the compact singular value decomposition (SVD) of $\mathbf{X}$ be $\mathbf{X}=\mathbf{U}\mathbf{\Sigma}\mathbf{V}^\top$, where $\mathbf{U}\in\mathbb{R}^{m\times r}$, $\mathbf{V}\in\mathbb{R}^{n\times r}$, $r=\operatorname{rank}(\mathbf{X})$, and $\mathbf{\Sigma}=\operatorname{diag}(\sigma_1,\ldots,\sigma_r)$ with $\sigma_1\geq\cdots\geq\sigma_r>0$. We use $\mathbf{U}_{[:,1:k]}$ and $\mathbf{V}_{[:,1:k]}$ to denote the top-$k$ dominant singular vectors, with $k\leq r$.

\paragraph{Zeroth-order (ZO) optimization}
We consider the generic optimization problem $ \min_{\mathbf{X} \in \mathbb{R}} f(\mathbf{X})$, where $f(\cdot)$ is the training objective function and $\mathbf{X}=\{X_\ell\}_{\ell=1}^L$ denotes the trainable parameters. In deep learning, we express hidden-layer matrix parameters as $X_\ell \in \mathbb{R}^{m_\ell \times n_\ell}$, where $\ell$ denotes the number of layers. ZO methods typically adopt randomized gradient estimation (RGE), where the perturbation matrix is sampled either from a Gaussian distribution or uniformly from the sphere \citep{duchi2015optimal,nesterov2017random}. For a matrix weight parameter $\mathbf{X} \in \mathbb{R}^{m \times n}$, the generic form of RGE estimator is
\begin{equation}
\hat{\nabla}_{\mathbf X} f(\mathbf{X};\xi)
:=
\frac{1}{N}\sum_{i=1}^{N}\frac{F(\mathbf{X}+\mu \mathbf{Z}^{i};\xi)-F(\mathbf{X}-\mu \mathbf{Z}^{i};\xi)}{2\mu}\mathbf{Z}^{i},
\label{eq:rge_estimator}
\end{equation}
where $\mathbf{Z}_{i} \in \mathbb{R}^{m \times n}$ is the $i$-th random perturbation matrix, $\mu > 0$ is the smoothing parameter, and $N$ is the number of queries. Since $\mathbf{Z}_{i}$ is sampled in the ambient space, the ZO gradients are nearly full-rank, which is poorly matched to the low-rank structure of FO gradients often observed in LLM fine-tuning \citep{chen2024enhancing,huang2026spectra}. Furthermore, the full-space perturbation causes high variance in ZO gradient estimator, which scales as $\mathcal{O}(mn/N)$~\citep{nesterov2017random,liu2020primer}, causing slow convergence in ZO methods.

\paragraph{Subspace ZO estimation}
To solve the problem, a growing body of work has focused on subspace ZO methods. In ZO gradient estimator, they replace full-rank perturbations with low-rank matrix perturbations~\citep{chen2024enhancing,yu2025zeroth,lang2026powering}. The key to such modification is to explicitly decompose our random perturbation $\mathbf{Z}_{i}$ as $\mathbf{Z}_{i} = \mathbf{A}\mathbf{B}^{(i)}$, where both $\mathbf{A} \in \mathbb{R}^{m \times r}$ and $\mathbf{B}^{(i)} \in \mathbb{R}^{r \times n}$ are low-rank matrices at most rank $r$. These matrices could be obtained from the Gaussian sampling~\citep{chen2024enhancing} or singular space decomposition~\citep{yu2025zeroth,lang2026powering}, with $r \ll \min\{m,n\}$. Therefore the overall perturbation $\mathbf{A}\mathbf{B}^{(i)}$ is low-rank, which ensures that the variance of ZO estimation now scales with $r$ instead of $m$ or $n$~\citep{lang2026powering}.

\paragraph{Muon}
Muon~\citep{jordan2024muon} is a FO optimizer designed for matrix-valued neural network parameters. For a matrix parameter $\mathbf{X}_t$ and its stochastic gradient $\mathbf{G}_t$, Muon forms a momentum update and then orthogonalizes it before applying the parameter update:
$
    \mathbf{M}_t = \beta \mathbf{M}_{t-1} + (1-\beta)\mathbf{G}_t,
    \
    \mathbf{X}_{t+1} = \mathbf{X}_t - \eta_t \operatorname{msign}(\mathbf{M}_t).
$
Here $\operatorname{msign}(\cdot)$ denotes the matrix sign, or equivalently the polar factor of a matrix. Specifically, if $\mathbf{M}_t=\mathbf{U}_t\mathbf{\Sigma}_t\mathbf{V}_t^\top$ is the compact SVD of $\mathbf{M}_t$, then
\begin{equation}
    \operatorname{msign}(\mathbf{M}_t) := \mathbf{U}_t \mathbf{V}_t^\top.
\label{eq:msign_def}
\end{equation}
Thus, Muon preserves the singular directions of the momentum while replacing its nonzero singular values by one. In practice, Muon approximates $\operatorname{msign}(\mathbf{M}_t)$ using a small number of Newton--Schulz (NS) iterations, avoiding explicit SVD computation.


\section{Method}
\label{sec:method}


In this section, we present our complete algorithm, ZO-MOPI, which combines ZO subspace estimation, projected momentum, and Streaming Power Iteration (SPI). Our version of SPI is inspired by the cached power-iteration SVD in \cite{huang2026spectra}, and many variants of this classic idea have been applied as common orthogonalization methods for efficient FO spectral optimization \citep{liu2025cosmos,ahn2025dion,huang2026spectra,gong2026aronewlensmatrix,kexuefm-11654}. In particular, \cite{kexuefm-11654} coined the term "streaming power iteration", so we also adopt the same naming here to ensure consistency. However, in the ZO setting, we introduce two additional practical components that make SPI effective: projected momentum, which stabilizes the tracked spectral subspace over time, and lazy subspace sampling, which preserves the temporal continuity needed by the warm-started iteration. The detailed procedures for ZO-MOPI and SPI are described in Algorithm~\ref{alg:main_optimizer} and Algorithm~\ref{alg:streaming_pi}, respectively.

Before we discuss the technical details, we first hope to highlight the motivation behind our approach. Figure~\ref{fig:motivation_singular_values} highlights why ZO optimization needs a different kind of spectral acceleration. In LLM fine-tuning, FO gradients often exhibit a compact spectrum, with most energy concentrated in a few dominant singular directions. However, ZO estimates spread substantial energy into the spectral tail due to Monte Carlo noise. Thus, the issue is not merely that ZO gradients are noisy, but that their noise is spectrally structured in a way that can mislead FO-style spectral optimizers.


To address this issue, we implement a specific kind of SPI that focuses only on the top-$k$ dominant singular directions. Following the msign notation from Section~\ref{sec:preliminaries}, we define:
\begin{equation}
    \operatorname{msign}_k(\mathbf{G}) = \mathbf{U}_{[:,1:k]}\mathbf{V}_{[:,1:k]}^{\top},
\end{equation}
where $\mathbf{U}_{[:,1:k]}$ and $\mathbf{V}_{[:,1:k]}$ denote the top-$k$ left and right singular vectors. Therefore, the parameter $k$ becomes a hyper-parameter that can be tuned based on how reliable the gradient estimate is (ablations in Appendix~\ref{app:hyper_study}). Combined with subspace ZO estimation and projected momentum, SPI operates on smoother, lower-variance signals, enabling stable spectral acceleration without amplifying ZO noise.


\paragraph{How the ZO estimator is constructed.}
To preliminarily reduce variance and filter out non-informative noise in ZO estimates, we integrate SPI with a subspace ZO estimator. Recall the formulation of ZO estimator in Equation~\ref{eq:rge_estimator} and we modify it as follows:
\begin{equation}
    \hat{\mathbf{G}}_t = \frac{1}{N}\sum_{i=1}^{N}\frac{F(\mathbf{X}_t + \mu \mathbf{A}\mathbf{B}_t^{(i)};\xi) - F(\mathbf{X}_t - \mu \mathbf{A}\mathbf{B}_t^{(i)};\xi)}{2\mu}\mathbf{B}_t^{(i)}, 
\label{eq:spi_update_multi}
\end{equation}
where $\mathbf{A} \in \mathbb{R}^{m \times r}$ denote the current subspace matrix resampled every $\nu$ iterations and  $\mathbf{B}_t^{(i)} \in \mathbb{R}^{r \times n}$ is the perturbation matrix sampled independently at each iteration from a Gaussian distribution. Let $\hat{\mathbf{G}}_t \in \mathbb{R}^{r \times n}$ denote the corresponding subspace gradient estimator produced in the space of $\mathbf{B}_t^{(i)}$. In our method, for memory efficiency, we do not accumulate momentum in the full $\mathbb{R}^{m \times n}$ space, and instead use the $\mathbb{R}^{r \times n}$ space by only averaging over $\hat{\mathbf{G}}_t$:
\begin{equation}
    \mathbf{M}_t = \beta \mathbf{M}_{t-1} + (1-\beta)\hat{\mathbf{G}}_t,
\label{eq:subspace_momentum}
\end{equation}
where ${\mathbf{M}}_t \in \mathbb{R}^{r \times n}$ is the low-rank momentum maintained in the subspace, adding only $\mathcal{O}(rn)$ memory overhead. Finally, we apply SPI to extract only the top-$k$ dominant spectral directions from the momentum:
\begin{equation}
    \mathbf{O} = \mathrm{msign}_k(\mathbf{A}\mathbf{M}) = \mathbf{A}\,\mathrm{msign}_k(\mathbf{M}),
\label{eq:spi_practical}
\end{equation}
where we omit the iteration index $t$ for notational simplicity. The second equality follows from Proposition~\ref{prop:lossless_projection} in Appendix~\ref{app:proof_prop1}, enabling us to apply $\mathrm{msign}_k()$ directly in the reduced $\mathbb{R}^{r \times n}$ space. 

\paragraph{Streaming Power Iteration (SPI).}
The core mechanism of tracking dominant spectral subspaces cheaply via streaming or warm-started power iteration is a well-established and standard numerical technique in machine learning. In statistical representation learning, foundational Online PCA methods such as~\cite{Oja1982SimplifiedNM} utilize stochastic updates to estimate the principal eigenspaces of streaming data covariance.
More recently, this streaming SVD technique has been adapted for deep learning as gradient compression techniques~\citep{vogels2019powersgd} and efficient orthogonalization techniques~\citep{huang2026spectra,ahn2025dion,kexuefm-11654}. Fundamentally, all these methods belong to the same class of warm-started techniques, which rely on the assumption that the matrices being tracked change slowly over time.

Concretely, as shown in Algorithm~\ref{alg:streaming_pi}, we cache a rank-$k$ right singular subspace $\mathbf{V}_{t-1} \in \mathbb{R}^{n \times k}$ and use it to initialize the $t^{\text{th}}$ iteration. In each iteration, first, to obtain the updated $\mathbf{V}_{t}$, we project the covariance matrix $ \mathbf{M}_t^\top \mathbf{M}_t $ onto the previous right singular subspace $\mathbf{V}_{t-1}$ and ensure orthogonality of $\mathbf{V}_{t}$ via QR decomposition (lines 1-2 in Algorithm~\ref{alg:streaming_pi}):
\begin{equation}
    \mathbf{V}_t = \mathrm{QR}(\mathbf{M}_t^\top (\mathbf{M}_t \mathbf{V}_{t-1})) \in \mathbb{R}^{n \times k}.
\end{equation}
Then we compute the updated and normalized left singular subspace $\mathbf{U}_{t}$ by projecting the momentum onto the new right singular subspace (line 3):
\begin{equation}
    \mathbf{U}_t = \mathrm{NormalizeColumns}(\mathbf{M}_t \mathbf{V}_t) \in \mathbb{R}^{r \times k}.
\end{equation}
Finally, we construct the rank-$k$ partial orthogonalization:
\begin{equation}
    \mathrm{msign}_k(\mathbf{M}_t) := \mathbf{U}_t \mathbf{V}_t^\top,
\label{eq:spi_msign}
\end{equation}
which can be viewed as a partial orthogonalization of $\mathbf{M}_t$. Unlike NS iteration, which approximately equalizes the entire active spectrum of the gradient, SPI only preserves the top-$k$ dominant singular directions, which avoids amplifying noisy directions in ZO estimates. Additionally, this also makes the orthogonalization step cheaper. For a matrix of size $m \times n$, Though NS iteration is more efficient than SVD, it still requires matrix multiplications on the order of $\mathcal{O}(mn\min(m,n))$. And SPI only tracks a rank-$k$ factorization and costs $\mathcal{O}(mnk+nk^2)$ with $k \ll \min(m,n)$.

It might be concerning to some that we are warm-starting the iteration with the cached $\mathbf{V}_{t-1}$ from the previous step, and even when the singular spaces of $\mathbf{M}_t$ are different for each $t$. However, rather than posing an issue, this practice improves efficiency and generalization at the same time, since the dominant subspace evolves gradually under momentum smoothing. This temporal continuity means the previous right singular subspace $\mathbf{V}_{t-1}$ remains highly informative for $\mathbf{M}_t$, making it a reliable warm start. We also provide a theoretical characterization of the error in subspace tracking of $\mathbf{V}$, showcasing its stability (proof in Appendix~\ref{app:proof_convergence}, inspired by the analytical framework in \cite{kexuefm-11710}): 


\begin{lemma}
\label{lem:spi_approximation_error}

We define the tracking error $\mathcal{T}_t
:= \mathbf{V}_{\bot,t}^\top\mathbf{V}_t
\left(\mathbf{V}_{\star,t}^\top\mathbf{V}_t\right)^{-1}$
which is the tangent of the principal angles between the top-$k$ dominant subspace $\mathbf{V}_t$ computed by SPI and the true dominant subspace $\mathbf{V}_{\star,t}$. $\mathbf{V}_{\bot,t}$ denotes the orthogonal complement of  $\mathbf{V}_{\star,t}$ at iteration $t$,  representing the subspace spanned by the remaining $n-k$ singular vectors.

Under the smooth drift assumption, the bound of the tracking error $\mathcal{T}_t$ satisfies:
$
\|\mathcal{T}_t\|_2 \lesssim \gamma \|\mathcal{T}_{t-1}\|_2 + \gamma \delta,
\label{eq:spi_tracking_error_main}
$
where the spectral gap $\gamma := \sup_t (\sigma_{k+1, t} / \sigma_{k, t})^2 < 1$, and the true dominant subspace drifts smoothly across iterations such that $\|\mathbf{V}_{\bot,t}^\top \mathbf{V}_{\star,t-1}\|_2 \le \delta$ for a small drift constant $\delta > 0$.

\end{lemma}
\paragraph{Why projected momentum is needed beyond variance reduction.}
In ZO-MOPI, we use the subspace momentum instead of full-space momentum which avoids additional memory overhead. This requires additional operation for momentum refresh when the subspace is resampled (every $\nu$ iterations via lazy sampling)~\citep{chen2024enhancing}. Since the low-rank momentum $\mathbf{M}_t$ is accumulated relative to the basis defined by the current subspace $\mathbf{A}^{\mathrm{old}}$, we must project it to the new basis spanned by the refreshed subspace $\mathbf{A}^{\mathrm{new}}$ (lines 4-5 in Algorithm~\ref{alg:main_optimizer}). 
\begin{wrapfigure}{r}{0.4\textwidth}
  \begin{center}
    \includegraphics[width=0.9\linewidth]{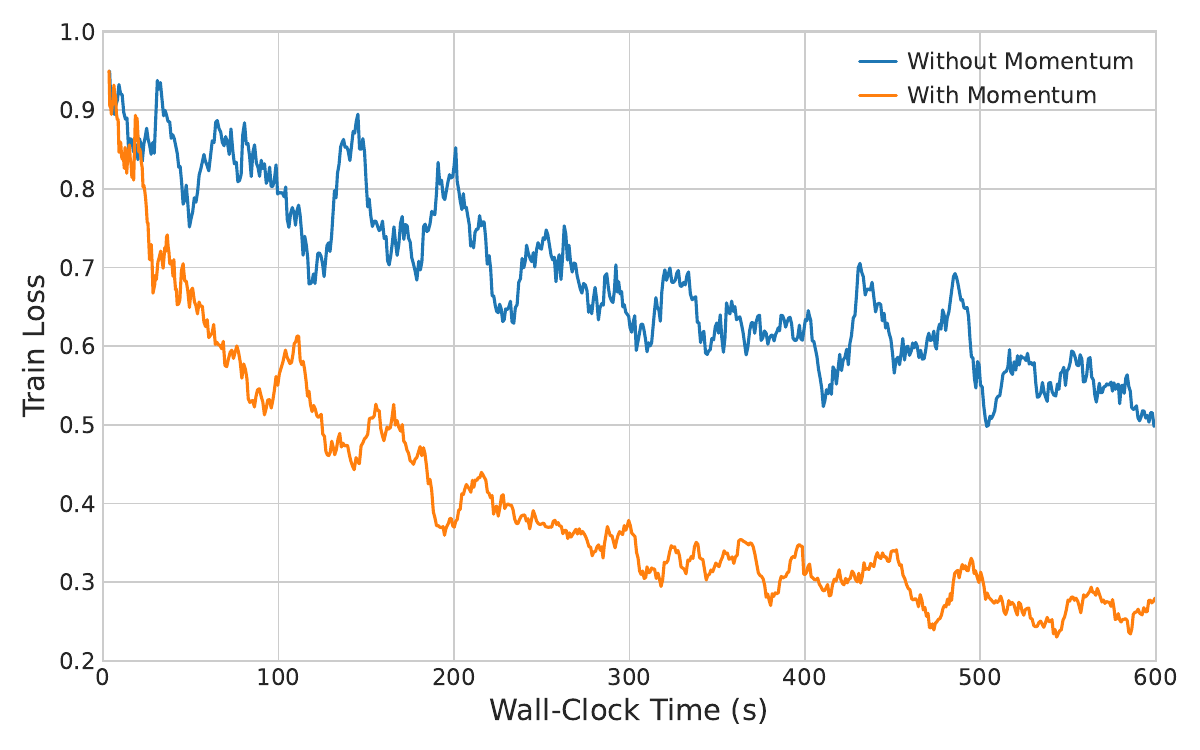}
  \end{center}
  \caption{ZO-MOPI's training loss with and without momentum.}
  \label{fig:hyper_momentum}
\end{wrapfigure}
The basis transformation is given by:
\begin{equation}
    \mathbf{M}_{t-1}^{\mathrm{new}} \leftarrow \frac{1}{m}(\mathbf{A}^{\mathrm{new}})^\top \mathbf{A}^{\mathrm{old}} \mathbf{M}_{t-1}^{\mathrm{old}},
\label{eq:projection_momentum}
\end{equation}
where $\mathbf{M}_{t-1}^{\mathrm{old}}, \mathbf{M}_{t-1}^{\mathrm{new}} \in \mathbb{R}^{r \times n}$ are the momentum representations in the old and new basis respectively. Momentum plays a very important role in our algorithm beyond reducing variance, since the effectiveness of SPI is contingent on it. This contrasts with the ablation studies in~\cite{lang2026powering}, as the authors found that momentum could not yield notable performance gains in their settings. However, since we have used the streaming version of power iteration, temporal continuity becomes a necessity, which can only be efficiently achieved via accumulating momentum~\citep{huang2026spectra,kexuefm-11710}. Without momentum, a single noisy ZO estimate $\hat{\mathbf{G}}_t$ can cause the dominant subspace to fluctuate drastically across iterations, which makes power iteration unstable and less effective.  Figure~\ref{fig:hyper_momentum} numerically demonstrates what will happen if we switch off momentum, justifying our algorithmic choice.


\paragraph{Lazy sampling strategy is equally crucial for SPI.}
\begin{wrapfigure}{l}{0.4\textwidth}
  \begin{center}
    \includegraphics[width=0.85\linewidth]{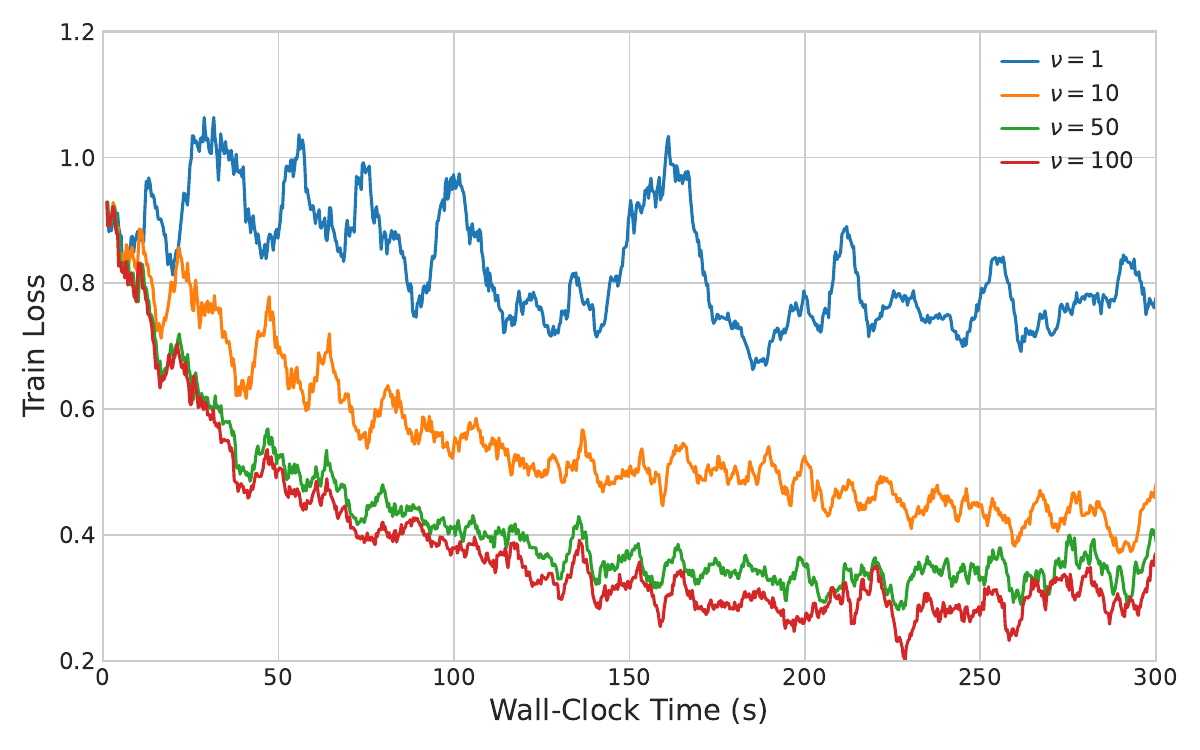}
  \end{center}
  \caption{Lazy sampling interval choices}
  \label{fig:hyper_nu}
\end{wrapfigure}
In our implementation, we fix the subspace $\mathbf{A}$ and resample it every $\nu$ iteration. We need periodic updates in $\mathbf{A}$ to encourage exploration, but shifting it too frequently may hinder the efficacy of SPI since it requires strong continuity. If the subspace changes too frequently, the spectral directions can vary substantially across steps, making the cached singular subspace $\mathbf{V}_{t-1}$ not useful and even misleading in estimating the updated $\mathbf{V}_{t}$.  As shown in Figure~\ref{fig:hyper_nu}, a small update interval $\nu$ can negatively impact the convergence and the training performance.

\begin{algorithm}[htbp]
\caption{ZO-MOPI: Zeroth-Order Momentum with Partial Orthogonalization}
\label{alg:main_optimizer}
\begin{algorithmic}[1]
\Require $\mathbf{X}_{t-1}, \mathbf{M}_{t-1}, \mathbf{A}, \mathbf{V}_{t-1}, \eta, \beta, \mu, r, k, \nu, N$
\Ensure $\mathbf{X}_t, \mathbf{M}_t, \mathbf{A}, \mathbf{V}_t$

\If{$t \bmod \nu = 0$}
\State $\mathbf{A}^{\mathrm{old}} \leftarrow \mathbf{A}$
\State $\mathbf{A}_{\mathrm{rand}} \sim \mathcal{N}(0, \mathbf{I}) \in \mathbb{R}^{m \times r}$
\State $\mathbf{A} \leftarrow \mathrm{QR}(\mathbf{A}_{\mathrm{rand}})$
\State $\mathbf{M}_{t-1} \leftarrow \frac{1}{m}\mathbf{A}^\top \mathbf{A}^{\mathrm{old}}\mathbf{M}_{t-1}$ \Comment{Project momentum onto the refreshed subspace}
\EndIf
\State $\hat{\mathbf{G}}_t \leftarrow \mathbf{0} \in \mathbb{R}^{r \times n}$
\For{$i = 1,\ldots,N$}
\State $\mathbf{B}_t^{(i)} \sim \mathcal{N}(0, \mathbf{I}) \in \mathbb{R}^{r \times n}$
\State $\hat{\mathbf{G}}_t \leftarrow \hat{\mathbf{G}}_t + \frac{\mathcal{L}(\mathbf{X}_{t-1} + \mu \mathbf{A}\mathbf{B}_t^{(i)}) - \mathcal{L}(\mathbf{X}_{t-1} - \mu \mathbf{A}\mathbf{B}_t^{(i)})}{2\mu N}\mathbf{B}_t^{(i)}$
\EndFor
\State $\mathbf{M}_t \leftarrow \beta \mathbf{M}_{t-1} + (1 - \beta)\hat{\mathbf{G}}_t$A
\State $\mathbf{O}_t, \mathbf{V}_t \leftarrow \text{StreamingPowerIteration}(\mathbf{M}_t, \mathbf{V}_{t-1}, k)$
\State $\mathbf{X}_t \leftarrow \mathbf{X}_{t-1} - \eta \mathbf{A}\mathbf{O}_t$
\State \Return $\mathbf{X}_t, \mathbf{M}_t, \mathbf{A}, \mathbf{V}_t$
\end{algorithmic}
\end{algorithm}

\begin{algorithm}[htbp]
\caption{Streaming Power Iteration (SPI)}
\label{alg:streaming_pi}
\begin{algorithmic}[1]
\Require $\mathbf{M}_t \in \mathbb{R}^{r \times n}$, $\mathbf{V}_{t-1} \in \mathbb{R}^{n \times k}$, $k$
\Ensure $\mathbf{O}_t\in \mathbb{R}^{r \times n}$, $\mathbf{V}_t \in \mathbb{R}^{n \times k}$

\State $\mathbf{Q} \leftarrow \mathbf{M}_{t}^\top (\mathbf{M}_t \mathbf{V}_{t-1}) \in \mathbb{R}^{n \times k}$ \Comment{Warm-started power iteration step}
\State $\mathbf{V}_t \leftarrow \mathrm{QR}(\mathbf{Q})$  
\State $\mathbf{U}_t \leftarrow \mathrm{NormalizeColumns}(\mathbf{M}_t \mathbf{V}_t) \in \mathbb{R}^{r \times k}$ \Comment{Column normalization}
\State $\mathbf{O}_t \leftarrow \mathbf{U}_{t} \mathbf{V}_t^\top$
\State \Return $\mathbf{O}_t, \mathbf{V}_t$
\end{algorithmic}
\end{algorithm}

\subsection{A Note on Convergence}
\label{sec:theory}
We establish theoretical convergence for the ZO-MOPI algorithm, showing that our new approach will not diverge easily. We adopt some standard assumptions in stochastic optimization to facilitate the proof. The details can be found in Appendix~\ref{app:proof_convergence}






\begin{theorem}
\label{thm:main_convergence}
Suppose the objective $f$ satisfies $L$-smoothness such that $\|\nabla f_i(\mathbf{X})-\nabla f_i(\mathbf{Y})\|_F \le L\|\mathbf{X}-\mathbf{Y}\|_F$ and the stochastic gradient variance is bounded by $\frac{1}{n}\sum_{i=1}^n
\|\nabla f_i(\mathbf{X})-\nabla f(\mathbf{X})\|_F^2 \le \sigma^2$. Further assume that the partial matrix sign operator is locally $L_m$-Lipschitz continuous, i.e., $\|msign_k(\mathbf{X}) - msign_k(\mathbf{Y})\|_F \le L_m \|\mathbf{X} - \mathbf{Y}\|_F$, and the momentum tracks the projected smoothed gradient $\mathbf{G}_t^\mu$ with mean squared error $\mathbb{E}[\|\mathbf{M}_t - \mathbf{G}_t^\mu\|_F^2] \le \sigma_M^2$. Then, under a learning rate $\eta = \Theta(\sqrt{(\Delta_0 + L\mu^2)/(LkT)})$, the sequence $\{\mathbf{X}_t\}$ generated by ZO-MOPI satisfies:
\begin{equation}
\frac{1}{T}\sum_{t=0}^{T-1}\mathbb{E}[\|\nabla f(\mathbf{X}_t)\|_F^2] \le \mathcal{O}\left( \frac{1}{\alpha}\sqrt{\frac{(\Delta_0 + L\mu^2)Lk}{T}} + \frac{L_m^2 \sigma_M^2}{\alpha^2} + \mu^2 L^2 d_{sub}^2 \right)
\end{equation}
where $\Delta_0 = f(\mathbf{X}_0) - f^*$ is the initial sub-optimality gap, $\alpha = \min_t \alpha_t \in (0, 1]$ is the spectral energy concentration constant, and $d_{sub} = rn$ is the effective dimension of the subspace perturbation.
\end{theorem}

\section{Experiments}
\label{sec:experiments}
\begin{figure}[t]
    \centering
    \setlength{\tabcolsep}{2pt} 
    \begin{tabular}{@{}cc@{}}
        
        \includegraphics[width=0.4\linewidth]{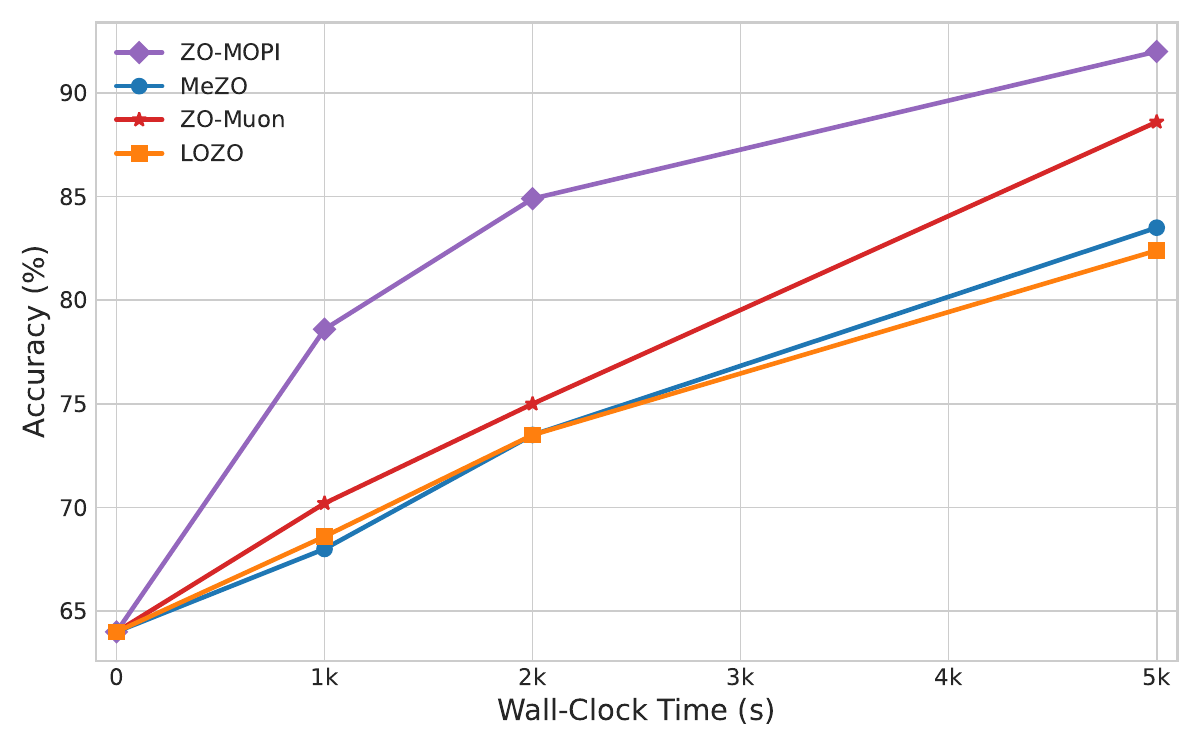} &
        \includegraphics[width=0.4\linewidth]{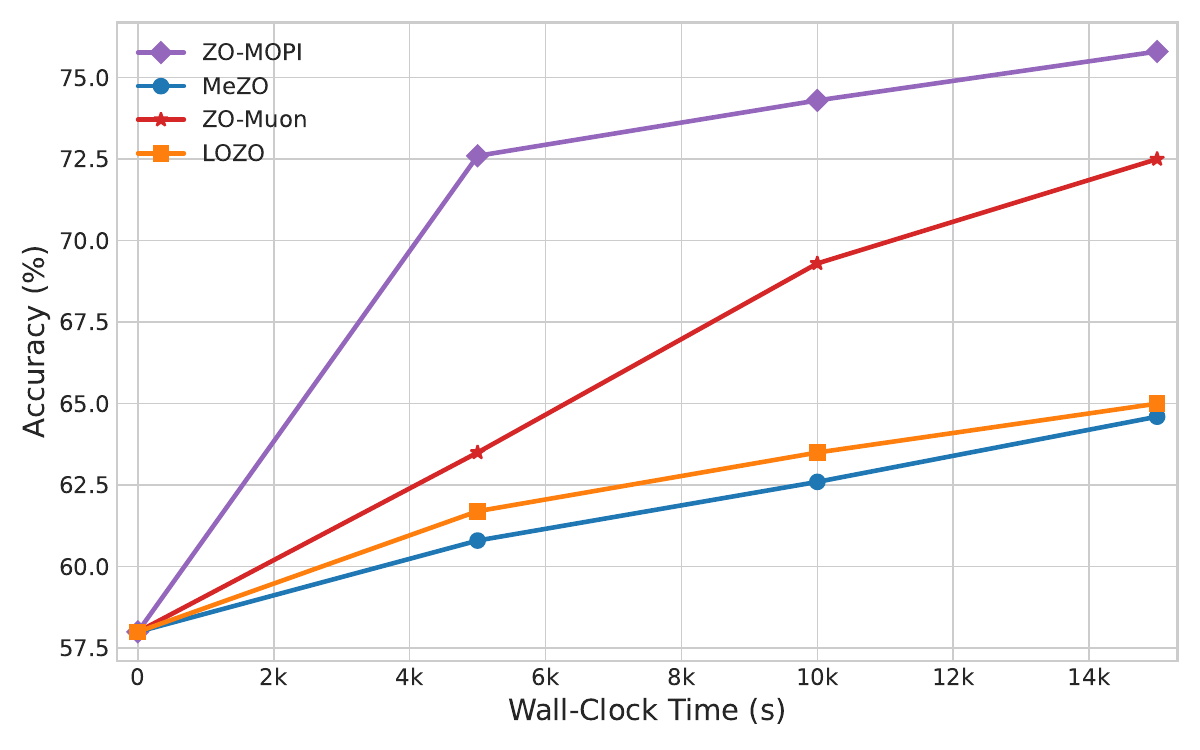} \\
        \small (a) SST-2 & \small (b) RTE \\[6pt] 

        \includegraphics[width=0.4\linewidth]{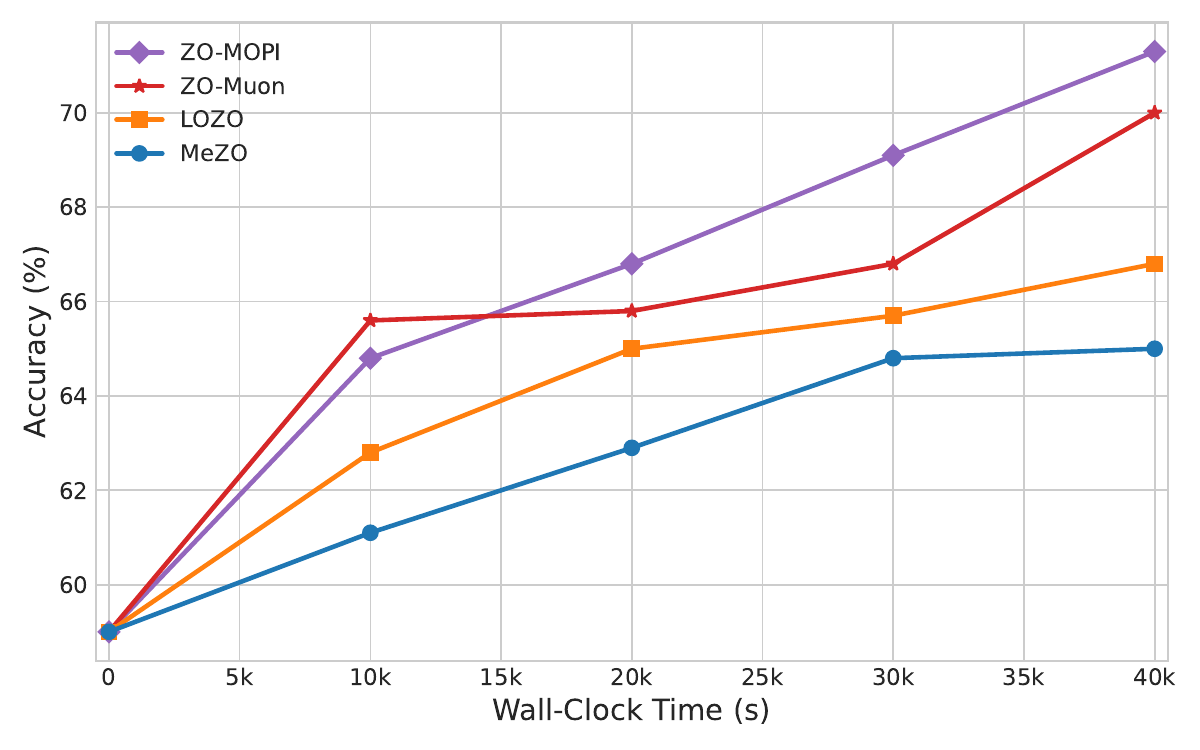} &
        \includegraphics[width=0.4\linewidth]{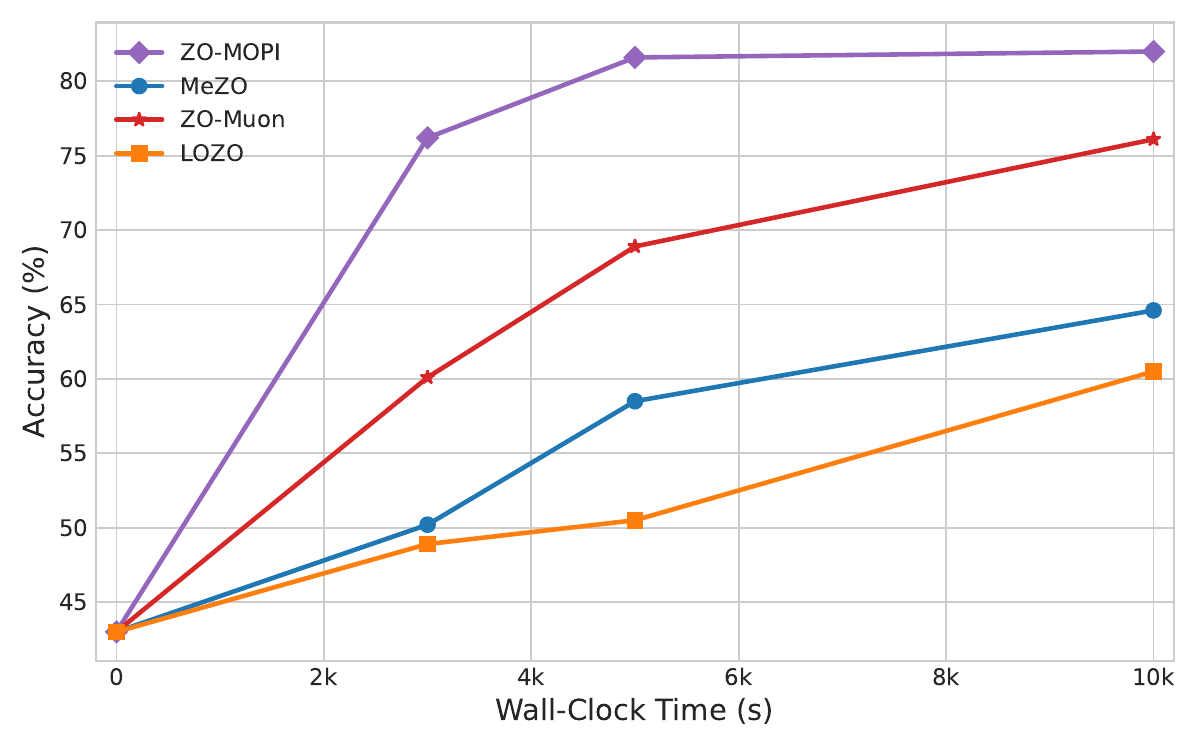} \\
        \small (c) BoolQ & \small (d) SQuAD

    \end{tabular}
    \caption{OPT-13B fine-tuning efficiency (Accuracy vs. Wall-clock time) across four SuperGLUE tasks. Each panel shows the accuracy improvement over time for the respective task.}
    \label{fig:opt13b_acc_vs_time}
\end{figure}
\begin{figure}[t]
    \centering
    \setlength{\tabcolsep}{2pt} 
    \begin{tabular}{@{}cc@{}}
        
        \includegraphics[width=0.4\linewidth]{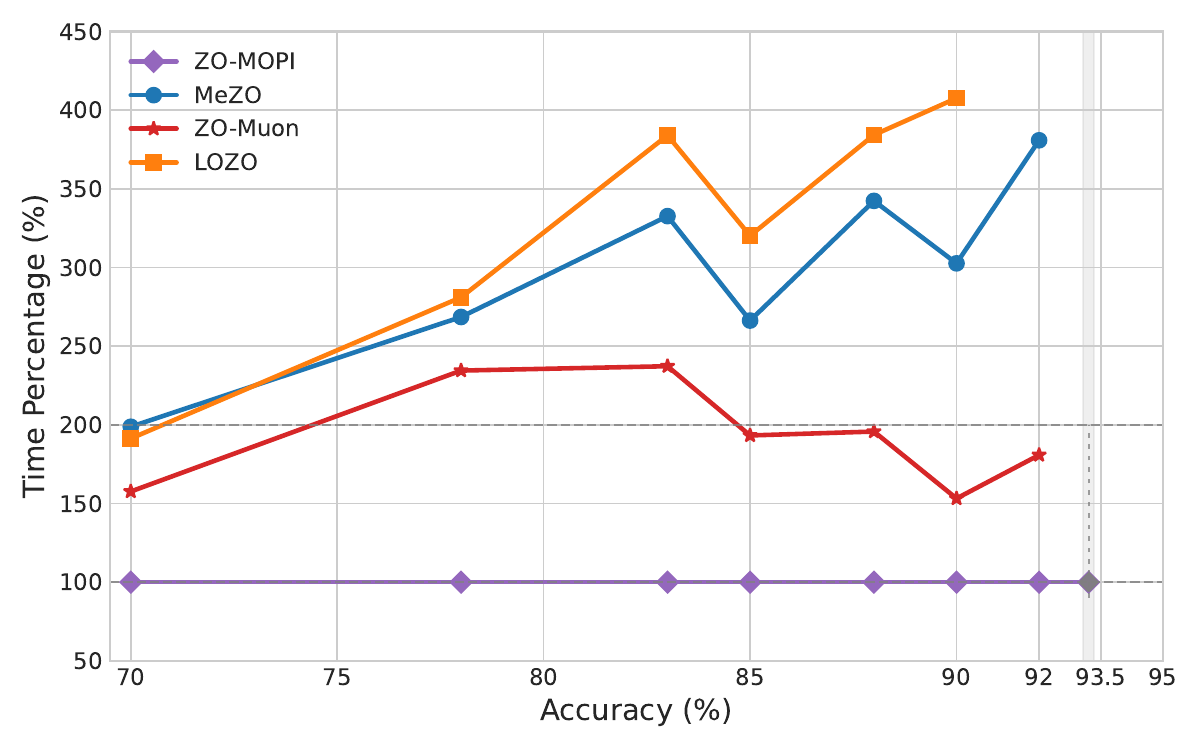} &
        \includegraphics[width=0.4\linewidth]{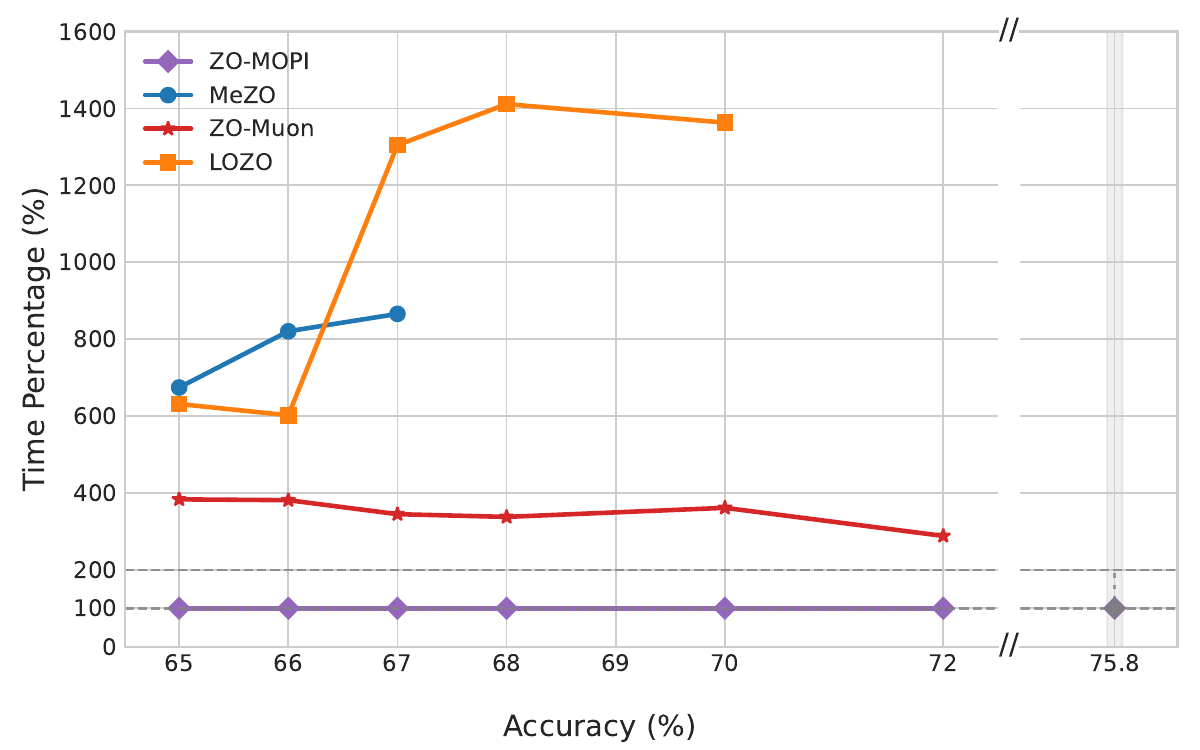} \\
        \small (a) SST-2 & \small (b) RTE \\[6pt] 

        \includegraphics[width=0.4\linewidth]{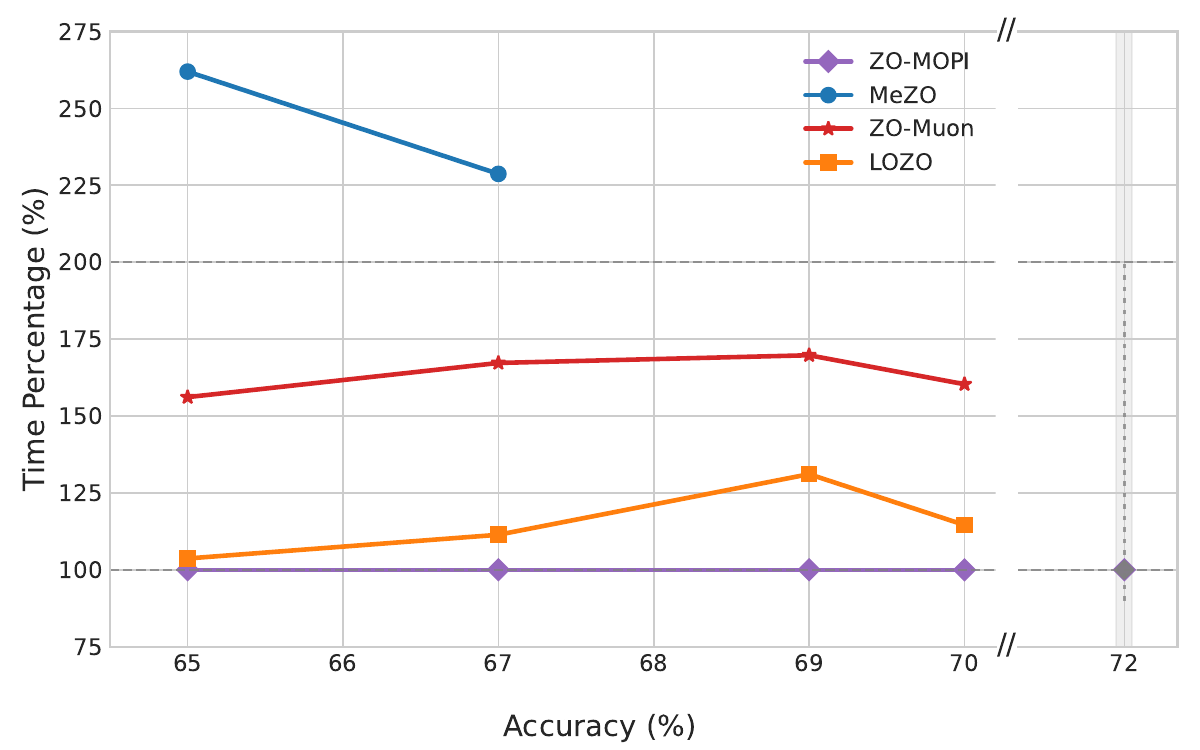} &
        \includegraphics[width=0.4\linewidth]{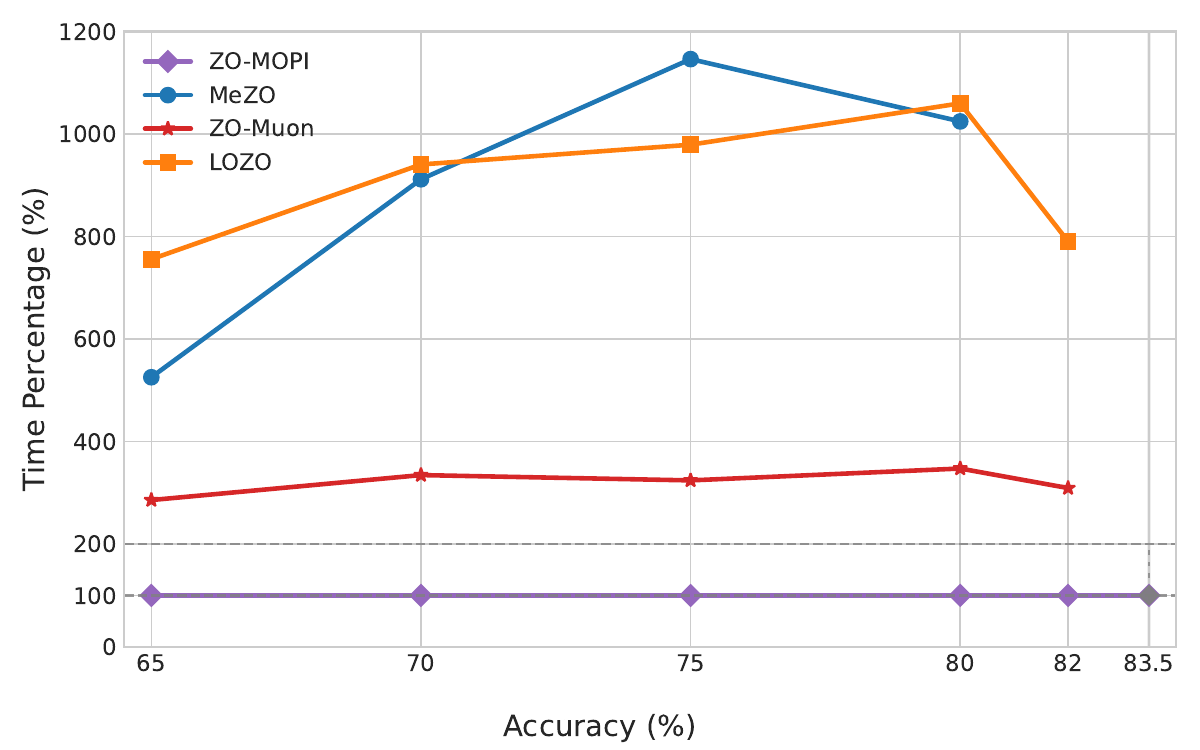} \\
        \small (c) BoolQ & \small (d) SQuAD

    \end{tabular}
    \caption{OPT-13B fine-tuning efficiency (Relative time-to-same-accuracy) across four SuperGLUE tasks. Each panel shows the time required relative to ZO-MOPI (normalized to $100\%$). Larger percentages indicate that a baseline requires more wall-clock time to reach the same accuracy.}
    \label{fig:opt13b_time_percentage}
\end{figure}
\begin{table*}[t]
    \centering
    \small
    \setlength{\tabcolsep}{6pt}
    \caption{Comparison of the wall-clock time (in hours) required to reach same accuracy thresholds on fine-tuning LLMs (OPT-13B and LLaMA3-8B) on SuperGLUE. The best ZO result for each task is highlighted in \textbf{bold}. Target thresholds are set to at least 85\% of the final performance achieved by Adam's baseline. Entries marked as ''N/A'' indicate that the method did not reach the target accuracy within the allotted training budget. }
    \label{tab:grouped_template_results}
    \begin{tabular}{lcccccccc}
        \toprule
        & \multicolumn{4}{c}{\textbf{OPT-13B}} & \multicolumn{4}{c}{\textbf{LLaMA3-8B}} \\
        \cmidrule(lr){2-5} \cmidrule(lr){6-9}
        \textbf{Method} & \textbf{SST-2} & \textbf{RTE} & \textbf{BoolQ} & \textbf{SQuAD} & \textbf{SST-2} & \textbf{RTE} & \textbf{BoolQ} & \textbf{SQuAD} \\
        \midrule
        MeZO & 3.57 & N/A & N/A & 9.01 & 3.57 & N/A & 8.56 & 6.85 \\
        LOZO & 4.62 & N/A & 15.53 & 6.86 & 2.72 & N/A & 10.45& 4.05  \\
        ZO-Muon & 1.74 & 3.28 & 14.59 & 2.35 & 1.64 & \textbf{4.05} & 4.08 & 2.85 \\
        \textbf{Our Method} & \textbf{1.13} & \textbf{1.29} & \textbf{10.34} & \textbf{0.83} & \textbf{1.05} & 5.06 & \textbf{3.54} & \textbf{2.68} \\
        \bottomrule
    \end{tabular}
\end{table*}
\begin{table*}[t]
    \centering
    \small
    \setlength{\tabcolsep}{3pt}
    \caption{We show the accuracy of ZO baselines for LLM (OPT-13B and LLaMA3-8B) fine-tuning on SuperGLUE. The best ZO fine-tuning result for each task is highlighted in \textbf{bold}. Values in parentheses indicate the relative percentage difference (\%) compared to our method.}
    \label{tab:main_llm_results}
    \begin{tabular}{@{}lcccccccc@{}}
        \toprule
        & \multicolumn{4}{c}{\textbf{OPT-13B}} & \multicolumn{4}{c}{\textbf{LLaMA3-8B}} \\
        \cmidrule(lr){2-5} \cmidrule(lr){6-9}
        \textbf{Method} & \textbf{SST-2} & \textbf{RTE} & \textbf{BoolQ} & \textbf{SQuAD} & \textbf{SST-2} & \textbf{RTE} & \textbf{BoolQ} & \textbf{SQuAD} \\
        \midrule
        Adam & 95.3 {\scriptsize (+2.3)} & 80.9 {\scriptsize (+6.7)} & 83.5 {\scriptsize (+16.0)} & 89.5 {\scriptsize (+7.2)} & 96.0 {\scriptsize (+3.3)} & 92.0 {\scriptsize (+17.9)} & 86.6 {\scriptsize (+5.0)} & 90.4 {\scriptsize (+2.1)} \\
        LoRA & 94.8 {\scriptsize (+1.7)} & 78.3 {\scriptsize (+3.3)} & 80.2 {\scriptsize (+11.4)} & 88.0 {\scriptsize (+5.4)} & 95.0 {\scriptsize (+2.3)} & 80.9 {\scriptsize (+3.7)} & 86.4 {\scriptsize (+4.7)} & 89.4 {\scriptsize (+1.0)} \\
        \midrule
        MeZO & 91.4 {\scriptsize (-1.9)} & 66.1 {\scriptsize (-12.8)} & 67.3 {\scriptsize (-6.5)} & 81.8 {\scriptsize (-2.0)} & 92.7 {\scriptsize (-0.2)} & 66.8 {\scriptsize (-14.4)} & 76.7 {\scriptsize (-7.0)} & 86.7 {\scriptsize (-2.0)} \\
        LOZO & 91.6 {\scriptsize (-1.7)} & 70.4 {\scriptsize (-7.1)} & 70.0 {\scriptsize (-2.8)} & \textbf{84.9} {\scriptsize (+1.7)} & 92.5 {\scriptsize (-0.4)} & 66.8 {\scriptsize (-14.4)} & 79.4 {\scriptsize (-3.8)} & \textbf{89.0} {\scriptsize (+0.6)} \\
        ZO-Muon & 92.5 {\scriptsize (-0.8)} & 72.9 {\scriptsize (-3.8)} & 71.3 {\scriptsize (-1.0)} & 84.2 {\scriptsize (+0.8)} & \textbf{93.7} {\scriptsize (+0.9)} & \textbf{78.0} {\scriptsize (0.0)} & \textbf{82.9} {\scriptsize (+0.5)} & 88.2 {\scriptsize (-0.3)} \\
        \textbf{ZO-MOPI} & \textbf{93.2} & \textbf{75.8} & \textbf{72.0} & 83.5 & 92.9 & \textbf{78.0} & 82.5 & 88.5 \\
        \bottomrule
    \end{tabular}
\end{table*}
\subsection{Experiment setups}
\label{subsec:experiment_setups}

We evaluate the performance of ZO optimization on LLM fine-tuning tasks. For LLMs, we primarily consider  OPT-13B~\citep{zhang2022opt} and LLaMA3-8B~\citep{grattafiori2024llama} on the SuperGLUE benchmark~\citep{wang2019superglue}. We compare the proposed ZO-MOPI with the following stateful ZO optimization baselines, which were primarily developed for large-model fine-tuning. These include MeZO~\citep{malladi2023fine},   LOZO~\citep{chen2024enhancing} and ZO-Muon~\citep{lang2026powering}. In addition to ZO approaches, we also include FO baselines, Adam~\citep{kingma2014adam} and LoRA (with Adam)~\citep{hu2022lora}.

\textbf{Implementation details.}  For our method, we fix the subspace rank to $r=64$ and set the spectral rank in SPI to $k=32$ throughout all LLM fine-tuning experiments. And we search over $N \in \{4, 8, 16\}$. To keep the comparison fair under matched query budgets, we scale the number of training steps inversely with $N$ (e.g. 8k steps for $N = 4$ and 4k steps for $N = 8$). The subspace $\mathbf{A}$ is updated lazily via QR decomposition rather than being resampled at every iteration. In particular, for our method $\mathbf{A}$ is refreshed every 500 iterations. We set the training batch size to 16 for all experiments. Additional experiment setups, including settings of baselines, are detailed in Appendix~\ref{app:details}.

\subsection{Experiment results}
\label{subsec:experiment_results}
        



\textbf{Our method delivers a stronger convergence and efficiency profile than prior ZO baselines.} Figure~\ref{fig:opt13b_acc_vs_time} and~\ref{fig:opt13b_time_percentage} show that ZO-MOPI reaches high-accuracy regions earlier than prior ZO baselines. The gain is especially clear against ZO-Muon: on RTE and SQuAD, ZO-Muon often needs around $3\times$ the wall-clock time of ZO-MOPI to reach the same accuracy, while on SST-2 it still requires about $1.5\times$--$2.4\times$ more time. Table~\ref{tab:grouped_template_results} shows that ZO-MOPI reaches the competitive accuracy threshold (at least 85\% of the Adam's baseline) with the shortest wall-clock time across most tasks. These results suggest that SPI-style partial orthogonalization provides a more efficient spectral update than NS-style full-spectrum orthogonalization in noisy ZO fine-tuning.

\textbf{Our method remains competitive with the strongest ZO baselines in final task performance while prioritizing speed.} As shown in Table~\ref{tab:main_llm_results}, our method attains the best ZO accuracy on some tasks like fine-tuning OPT-13B for SST-2 and RTE, improving over ZO-Muon from 92.5\% to 93.2\% on SST-2 and from 72.9\% to 75.8\% on RTE. ZO-MOPI outperforms LOZO in most tasks, which shows the integration of SPI with subspace ZO estimation can
improve the ZO optimization.


\section{Conclusion}
\label{sec:conclusion}
In this paper, we propose partial spectral orthogonalization to accelerate ZO optimization for LLMs by selectively extracting dominant singular directions while avoiding noise amplification. This mechanism is supported by subspace gradient estimation and projected momentum, which reduce estimation variance and stabilize spectral tracking. The experimental results demonstrate that our method significantly outperforms strong ZO baselines in efficiency and accuracy across most evaluated tasks while preserving memory advantages. Future work will extend these spectral techniques from parameter-efficient fine-tuning to general large-scale training regimes.

\nocite{team2024gemma,grattafiori2024llama,zhang2022opt,wang2019superglue,socher2013recursive,dagan2005pascal,de2019commitmentbank,clark2019boolq,pilehvar2019wic,rajpurkar2016squad,malladi2023fine,chen2024enhancing,yu2025zeroth,zhao2024second}

\bibliographystyle{plainnat}
\bibliography{neurips_2026}

\medskip

\clearpage
\appendix





\section{Analysis of Subspace Power Iteration}
\label{app:proof_prop1}
In Section 4, for computational efficiency, we operate the Streaming Power Iteration (SPI) in the reduced subspace rather than the full parameter space. In this section, we establish that this projection-based SPI is equivalent to its full-space counterpart. Our analysis is inspired by the proof of Proposition 1 in \cite{lang2026powering}. Let $\mathbf{A} \in \mathbb{R}^{m \times r}$ be a subspace matrix with orthonormal columns (i.e., $\mathbf{A}^\top \mathbf{A} = \mathbf{I}$). We define the full-space gradient at iteration $t$ as $\mathbf{G} = \mathbf{A}\mathbf{M} \in \mathbb{R}^{m \times n}$ and let its SVD be:
\begin{equation}
    \mathbf{G}
    =
    \mathbf{U}
    \operatorname{diag}(\boldsymbol{\sigma})
    \mathbf{V}^\top,
\end{equation}
and let $\mathbf{U}_{[:,1:k]}$ and $\mathbf{V}_{[:,1:k]}$ denote the top-$k$ left and right singular vectors. The power iteration operator is
\begin{equation}
    \operatorname{msign}_k(\mathbf{G})
    :=
    \mathbf{U}_{[:,1:k]}
    \mathbf{V}_{[:,1:k]}^\top.
\label{eq:appendix_msign_def}
\end{equation}

\begin{proposition}
\label{prop:lossless_projection}
If the subspace matrix is chosen as $\mathbf{A}\in\mathbb{R}^{m\times r}$, obtained from the SVD of $\mathbf{M}$, then the projection for rank-$k$ gradient orthogonalization is lossless. That is,
\begin{equation}
    \mathbf{A}\,\operatorname{msign}_k(\mathbf{A}^\top \mathbf{G})
    =
    \operatorname{msign}_k(\mathbf{G}).
\label{eq:appendix_prop1}
\end{equation}
\end{proposition}

\begin{proof}
The projected matrix satisfies
\begin{equation}
     \mathbf{A}^\top \mathbf{G}
    =
    \mathbf{A}^\top \mathbf{U}
    \operatorname{diag}(\boldsymbol{\sigma})
    \mathbf{V}^\top.
\end{equation}
Thus, taking the top-$k$ components, the matrix sign of the projected matrix is
\begin{equation}
     \operatorname{msign}_k(\mathbf{A}^\top \mathbf{G})
    =
    \mathbf{A}^\top \mathbf{U}_{[:,1:k]} \mathbf{V}_{[:,1:k]}^\top .
\end{equation}
Multiplying by $\mathbf{A}$ on the left gives
\begin{equation}
    \mathbf{A}\,\operatorname{msign}_k(\mathbf{A}^\top \mathbf{G})
    =
    \mathbf{U}_{[:,1:k]}
    \mathbf{V}_{[:,1:k]}^\top
    =
    \operatorname{msign}_k(\mathbf{G}),
\end{equation}
which proves \eqref{eq:appendix_prop1}. Consequently, by substituting $\mathbf{G} = \mathbf{AM}$ into this result and utilizing the property $\mathbf{A}^\top \mathbf{A} = \mathbf{I}$, we recover the practical implementation used in Equation~\eqref{eq:spi_practical}:
\begin{equation}
 \mathbf{A}\operatorname{msign}_k(\mathbf{A}^\top \mathbf{A}\mathbf{M})
 = \mathbf{A}\operatorname{msign}_k(\mathbf{M})
 = \operatorname{msign}_k(\mathbf{A}\mathbf{M}).
\end{equation}
This completes the proof.
\end{proof}

\section{Convergence Analysis}
\label{app:proof_convergence}

In this section, we present the convergence analysis of the ZO-MOPI algorithm and provide a detailed proof of Lemma~\ref{lem:spi_approximation_error} and Theorem~\ref{thm:main_convergence} in Section~\ref{sec:theory}.

\begin{lemma}
\label{lem:zo_smoothing_facts}
Suppose assumptions in Theorem~\ref{thm:main_convergence} hold, and define the smoothed objective $f_\mu(\mathbf{X})=\mathbb{E}_{\mathbf{u}\sim U_b}[f(\mathbf{X}+\mu\mathbf{u})]$, where $U_b$ is the uniform or Gaussian distribution over the unit Euclidean ball in dimension $d$, the ZO estimator yields:
\begin{equation}
\mathbb{E}_{\mathbf{u}}\!\left[\widehat{\nabla} f(\mathbf{X})\right]
=
\nabla f_\mu(\mathbf{X}).
\label{eq:zo_unbiased_smoothed_appendix}
\end{equation}
where $f_\mu$ is $L$-smooth. Moreover, for any $\mathbf{X}$,
\begin{equation}
|f_\mu(\mathbf{X})-f(\mathbf{X})|
\le
\frac{L\mu^2}{2},
\qquad
\|\nabla f_\mu(\mathbf{X})-\nabla f(\mathbf{X})\|_F^2
\le
\frac{\mu^2L^2d^2}{4},
\label{eq:zo_smoothing_bias_appendix}
\end{equation}
and
\begin{equation}
\frac{1}{2}\|\nabla f(\mathbf{X})\|_F^2
-\frac{\mu^2L^2d^2}{4}
\le
\|\nabla f_\mu(\mathbf{X})\|_F^2
\le
2\|\nabla f(\mathbf{X})\|_F^2
+\frac{\mu^2L^2d^2}{2}.
\label{eq:true_smoothed_gradient_relation}
\end{equation}
Furthermore, for the smoothed component objectives $f_{i,\mu}$,
\begin{equation}
\frac{1}{n}\sum_{i=1}^n
\|\nabla f_{i,\mu}(\mathbf{X})-\nabla f_\mu(\mathbf{X})\|_F^2
\le
\sigma^2 .
\label{eq:smoothed_variance_bound}
\end{equation}
\end{lemma}

The detailed proof of Lemma~\ref{lem:zo_smoothing_facts} can be found in \cite{liu2018zeroth}.

\subsection*{Proof of Lemma~\ref{lem:spi_approximation_error}}

\begin{proof}
Our analysis is inspired by~\cite{kexuefm-11710}. First, let $\mathbf{S}_t = \mathbf{M}_t^\top \mathbf{M}_t$ be the covariance matrix at iteration $t$, $\mathbf{V}_{\star,t}$ and $\mathbf{V}_{\bot,t}$ as the true top-$k$ right singular subspace and its orthogonal complement exactly at step $t$. The SVD of $\mathbf{S}_t$ can be expressed as:
\[
\mathbf{S}_t
=
\begin{bmatrix}
\mathbf{V}_{\star,t} & \mathbf{V}_{\bot,t}
\end{bmatrix}
\begin{bmatrix}
\boldsymbol{\Lambda}_{1,t} & \mathbf{0}\\
\mathbf{0} & \boldsymbol{\Lambda}_{2,t}
\end{bmatrix}
\begin{bmatrix}
\mathbf{V}_{\star,t} & \mathbf{V}_{\bot,t}
\end{bmatrix}^{\top},
\]
where $\boldsymbol{\Lambda}_{1,t}=\operatorname{diag}(\sigma_{1,t}^2,\ldots,\sigma_{k,t}^2)$ and $\boldsymbol{\Lambda}_{2,t}=\operatorname{diag}(\sigma_{k+1,t}^2,\ldots,\sigma_{r,t}^2)$.

Following the SPI procedure, the subspace $\mathbf{V}_t$ is updated through the projection and QR decomposition (lines 1-2 in Algorithm~\ref{alg:streaming_pi}):
\begin{equation}
\mathbf{V}_t,\mathbf{R}_t\leftarrow \mathrm{QR}(\mathbf{S}_t\mathbf{V}_{t-1}).
\end{equation}
Equivalently, this yields:
\begin{equation}
\mathbf{V}_t\mathbf{R}_t = \mathbf{S}_t\mathbf{V}_{t-1}.
\label{eq:spi_qr_equivalence}
\end{equation}

We define $\mathbf{C}_t=\mathbf{V}_{\star,t}^\top\mathbf{V}_t$ and $\mathbf{D}_t=\mathbf{V}_{\bot,t}^\top\mathbf{V}_t$, so that the tracking error at step $t$ is exactly $\mathcal{T}_t=\mathbf{D}_t\mathbf{C}_t^{-1}$. Projecting \eqref{eq:spi_qr_equivalence} onto the current true bases $\mathbf{V}_{\star,t}$ and $\mathbf{V}_{\bot,t}$ gives:
\begin{equation}
\mathbf{C}_t\mathbf{R}_t = \boldsymbol{\Lambda}_{1,t}\mathbf{V}_{\star,t}^\top\mathbf{V}_{t-1},
\qquad
\mathbf{D}_t\mathbf{R}_t = \boldsymbol{\Lambda}_{2,t}\mathbf{V}_{\bot,t}^\top\mathbf{V}_{t-1}.
\label{eq:spi_projected_recursion_dynamic}
\end{equation}

Combining $\mathbf{I} = \mathbf{V}_{\star,t-1}\mathbf{V}_{\star,t-1}^\top + \mathbf{V}_{\bot,t-1}\mathbf{V}_{\bot,t-1}^\top$
and \eqref{eq:spi_projected_recursion_dynamic}, we get:
\begin{equation}
\mathbf{V}_{t-1} = \mathbf{I}\mathbf{V}_{t-1} = \mathbf{V}_{\star,t-1}(\mathbf{V}_{\star,t-1}^\top\mathbf{V}_{t-1}) + \mathbf{V}_{\bot,t-1}(\mathbf{V}_{\bot,t-1}^\top\mathbf{V}_{t-1}) = \mathbf{V}_{\star,t-1}\mathbf{C}_{t-1} + \mathbf{V}_{\bot,t-1}\mathbf{D}_{t-1}.
\label{eq:v_decomposition}
\end{equation}


Substituting  \eqref{eq:v_decomposition} the right-hand side of \eqref{eq:spi_projected_recursion_dynamic}, we define the basis transition matrices $\boldsymbol{\Delta}_{ij}^{(t)}$:
\begin{align}
\mathbf{V}_{\star,t}^\top\mathbf{V}_{t-1} 
&= \underbrace{\mathbf{V}_{\star,t}^\top\mathbf{V}_{\star,t-1}}_{\boldsymbol{\Delta}_{11}^{(t)}}\mathbf{C}_{t-1} + \underbrace{\mathbf{V}_{\star,t}^\top\mathbf{V}_{\bot,t-1}}_{\boldsymbol{\Delta}_{12}^{(t)}}\mathbf{D}_{t-1} 
= \left(\boldsymbol{\Delta}_{11}^{(t)} + \boldsymbol{\Delta}_{12}^{(t)}\mathcal{T}_{t-1}\right)\mathbf{C}_{t-1}, \label{eq:delta_expansion_1} \\
\mathbf{V}_{\bot,t}^\top\mathbf{V}_{t-1} 
&= \underbrace{\mathbf{V}_{\bot,t}^\top\mathbf{V}_{\star,t-1}}_{\boldsymbol{\Delta}_{21}^{(t)}}\mathbf{C}_{t-1} + \underbrace{\mathbf{V}_{\bot,t}^\top\mathbf{V}_{\bot,t-1}}_{\boldsymbol{\Delta}_{22}^{(t)}}\mathbf{D}_{t-1} 
= \left(\boldsymbol{\Delta}_{21}^{(t)} + \boldsymbol{\Delta}_{22}^{(t)}\mathcal{T}_{t-1}\right)\mathbf{C}_{t-1}. \label{eq:delta_expansion_2}
\end{align}

Substituting \eqref{eq:delta_expansion_1} and \eqref{eq:delta_expansion_2} back into the projected equations \eqref{eq:spi_projected_recursion_dynamic}, we obtain:
\begin{equation}
\mathbf{C}_t\mathbf{R}_t = \boldsymbol{\Lambda}_{1,t}\left(\boldsymbol{\Delta}_{11}^{(t)} + \boldsymbol{\Delta}_{12}^{(t)}\mathcal{T}_{t-1}\right)\mathbf{C}_{t-1},
\label{eq:c_r_substituted}
\end{equation}
\begin{equation}
\mathbf{D}_t\mathbf{R}_t = \boldsymbol{\Lambda}_{2,t}\left(\boldsymbol{\Delta}_{21}^{(t)} + \boldsymbol{\Delta}_{22}^{(t)}\mathcal{T}_{t-1}\right)\mathbf{C}_{t-1}.
\label{eq:d_r_substituted}
\end{equation}

From equation \eqref{eq:c_r_substituted}, we isolate $\mathbf{R}_t$:
\begin{equation}
\mathbf{R}_t = \mathbf{C}_t^{-1}\boldsymbol{\Lambda}_{1,t}\left(\boldsymbol{\Delta}_{11}^{(t)} + \boldsymbol{\Delta}_{12}^{(t)}\mathcal{T}_{t-1}\right)\mathbf{C}_{t-1},
\end{equation}
which implies:
\begin{equation}
\mathbf{R}_t^{-1} = \mathbf{C}_{t-1}^{-1}\left(\boldsymbol{\Delta}_{11}^{(t)} + \boldsymbol{\Delta}_{12}^{(t)}\mathcal{T}_{t-1}\right)^{-1}\boldsymbol{\Lambda}_{1,t}^{-1}\mathbf{C}_t.
\end{equation}

Substituting $\mathbf{R}_t^{-1}$ into equation \eqref{eq:d_r_substituted} yields:
\[
\mathbf{D}_t = \boldsymbol{\Lambda}_{2,t}\left(\boldsymbol{\Delta}_{21}^{(t)} + \boldsymbol{\Delta}_{22}^{(t)}\mathcal{T}_{t-1}\right)\mathbf{C}_{t-1} \cdot \mathbf{C}_{t-1}^{-1}\left(\boldsymbol{\Delta}_{11}^{(t)} + \boldsymbol{\Delta}_{12}^{(t)}\mathcal{T}_{t-1}\right)^{-1}\boldsymbol{\Lambda}_{1,t}^{-1}\mathbf{C}_t.
\]
The terms $\mathbf{C}_{t-1}$ and $\mathbf{C}_{t-1}^{-1}$ cancel out. Multiplying by $\mathbf{C}_t^{-1}$ on the right gives the explicit recurrence relation for the tracking error:
\begin{equation}
\mathcal{T}_t = \boldsymbol{\Lambda}_{2,t}\left(\boldsymbol{\Delta}_{21}^{(t)} + \boldsymbol{\Delta}_{22}^{(t)}\mathcal{T}_{t-1}\right)\left(\boldsymbol{\Delta}_{11}^{(t)} + \boldsymbol{\Delta}_{12}^{(t)}\mathcal{T}_{t-1}\right)^{-1}\boldsymbol{\Lambda}_{1,t}^{-1}.
\label{eq:exact_tracking_error_recurrence}
\end{equation}

Now, we take the spectral norm ($\|\cdot\|_2$) on both sides of \eqref{eq:exact_tracking_error_recurrence}. Recall that the spectral gap implies $\|\boldsymbol{\Lambda}_{2,t}\|_2\|\boldsymbol{\Lambda}_{1,t}^{-1}\|_2 \le \gamma < 1$. Under the smooth drift assumption, the basis transitions are bounded by the drift constant $\delta$: $\|\boldsymbol{\Delta}_{21}^{(t)}\|_2 \le \delta$ and $\|\boldsymbol{\Delta}_{12}^{(t)}\|_2 \le \delta$. Due to the orthonormality of the bases, we also have $\sigma_{\min}(\boldsymbol{\Delta}_{11}^{(t)}) \ge \sqrt{1-\delta^2}$ and $\|\boldsymbol{\Delta}_{22}^{(t)}\|_2 \le 1$. Applying these bounds yields:
\begin{equation}
\|\mathcal{T}_t\|_2 
\le \gamma \frac{\|\boldsymbol{\Delta}_{21}^{(t)} + \boldsymbol{\Delta}_{22}^{(t)}\mathcal{T}_{t-1}\|_2}{\sigma_{\min}\left(\boldsymbol{\Delta}_{11}^{(t)} + \boldsymbol{\Delta}_{12}^{(t)}\mathcal{T}_{t-1}\right)} 
\le \gamma \frac{\|\mathcal{T}_{t-1}\|_2 + \delta}{\sqrt{1-\delta^2} - \delta\|\mathcal{T}_{t-1}\|_2}.
\end{equation}

For a sufficiently small drift $\delta$ and previous tracking error $\|\mathcal{T}_{t-1}\|_2$, the denominator is well-behaved and close to 1. Thus, we can approximate the recurrence mapping as:
\begin{equation}
\|\mathcal{T}_t\|_2 \lesssim \gamma \|\mathcal{T}_{t-1}\|_2 + \gamma\delta.
\end{equation}
This completes the proof. This result demonstrates that under the smooth drift assumption, the tracking error remains bounded by $\mathcal{O}(\frac{\gamma\delta}{1-\gamma})$, confirming that SPI can stably track the evolving dominant subspace throughout the optimization process.
\end{proof}

\subsection*{Proof of Theorem \ref{thm:main_convergence}}

\textit{Proof.} Since $f_\mu$ is $L$-smooth, $A_t^\top A_t = I$, and $O_t$ has rank at most $k$ with singular values bounded by one, taking the expectation yields:
\begin{align}
\mathbb{E}_t[f_\mu(X_{t+1})] &\le f_\mu(X_t) - \eta\mathbb{E}_t[\langle \nabla f_\mu(X_t), A_t O_t \rangle] + \frac{L\eta^2}{2}\mathbb{E}_t[||A_t O_t||_F^2] \notag \\
&\le f_\mu(X_t) - \eta\mathbb{E}_t[\langle \nabla f_\mu(X_t), A_t O_t \rangle] + \frac{L\eta^2 k}{2}. \label{eq:descent}
\end{align}

The inner product term can be decomposed as:
\begin{equation} \label{eq:inner_prod_decomp}
\mathbb{E}_t[\langle \nabla f_\mu(X_t), A_t O_t \rangle] = \mathbb{E}_t[\langle G_t^\mu, O_t \rangle] = \langle G_t^\mu, O_t^* \rangle + \mathbb{E}_t[\langle G_t^\mu, O_t - O_t^* \rangle]
\end{equation}
where $G_t^\mu = A_t^\top \nabla f_\mu(X_t)$ is the true smoothed gradient projected into the subspace, and $O_t^* = msign_k(G_t^\mu)$. Let $\sigma_{1,t} \ge \dots \ge \sigma_{r,t}$ be the singular values of $G_t^\mu$, and define
\begin{equation}
\alpha_t := \frac{\sum_{i=1}^k \sigma_{i,t}}{||\nabla f_\mu(X_t)||_F^2}, \quad \alpha := \min_{0 \le t \le T-1} \alpha_t \in [0,1].
\end{equation}

By definition, we bound $\langle G_t^\mu, O_t^* \rangle$:
\begin{equation} \label{eq:alpha_bound}
\langle G_t^\mu, O_t^* \rangle = \sum_{i=1}^k \sigma_{i,t} = \alpha_t ||\nabla f_\mu(X_t)||_F^2 \ge \alpha ||\nabla f_\mu(X_t)||_F^2.
\end{equation}

Substituting \eqref{eq:alpha_bound} into \eqref{eq:inner_prod_decomp} and applying Young's inequality gives:
\begin{align}
\mathbb{E}_t[\langle \nabla f_\mu(X_t), A_t O_t \rangle] &\ge \alpha ||\nabla f_\mu(X_t)||_F^2 - \frac{\alpha}{2}||G_t^\mu||_F^2 - \frac{1}{2\alpha}\mathbb{E}_t[||O_t - O_t^*||_F^2] \notag \\
&\ge \frac{\alpha}{2} ||\nabla f_\mu(X_t)||_F^2 - \frac{1}{2\alpha}\mathbb{E}_t[||O_t - O_t^*||_F^2]. \label{eq:youngs_ineq}
\end{align}
In the second inequality, we used $||G_t^\mu||_F^2 = ||A_t^\top \nabla f_\mu(X_t)||_F^2 \le ||\nabla f_\mu(X_t)||_F^2$.

Next, we bound the orthogonalization error $\mathbb{E}_t[||O_t - O_t^*||_F^2]$. Since SPI tracks the dominant subspace of the momentum $M_t$, we have $O_t \approx msign_k(M_t)$. By the assumption in Theorem~\ref{thm:main_convergence}, the partial orthogonalization is $L_m$-Lipschitz continuous, which yields:
\begin{equation}
\mathbb{E}_t[||O_t - O_t^*||_F^2] \le L_m^2 \mathbb{E}_t[||M_t - G_t^\mu||_F^2].
\end{equation}
Because the random perturbations $B_t^{(i)}$ in our ZO estimator are strictly constrained within the $r \times n$ subspace, the effective search dimension is $d_{sub} = rn$. Consequently, the variance of the tracked momentum $\sigma_M^2 := \mathbb{E}_t[||M_t - G_t^\mu||_F^2]$ is bounded by the standard ZO variance in the reduced space:
\begin{equation} \label{eq:error_bound}
\mathbb{E}_t[||O_t - O_t^*||_F^2] \le L_m^2 \sigma_M^2 \le \mathcal{O}\left( L_m^2 d_{sub}\sigma^2 + L_m^2 \mu^2 L^2 d_{sub}^2 \right).
\end{equation}

Substituting \eqref{eq:youngs_ineq} and \eqref{eq:error_bound} back into \eqref{eq:descent}:
\begin{equation}
\mathbb{E}_t[f_\mu(X_{t+1})] \le f_\mu(X_t) - \frac{\alpha\eta}{2} ||\nabla f_\mu(X_t)||_F^2 + \frac{\eta L_m^2}{2\alpha} \sigma_M^2 + \frac{L\eta^2 k}{2}.
\end{equation}

Taking the total expectation, summing over $t = 0, \dots, T-1$, and using $f_\mu(X_T) \ge f_\mu^*$ yields:
\begin{equation} \label{eq:smoothed_sum}
\frac{1}{T}\sum_{t=0}^{T-1}\mathbb{E}[||\nabla f_\mu(X_t)||_F^2] \le \frac{2\Delta_\mu}{\alpha\eta T} + \frac{L\eta k}{\alpha} + \frac{L_m^2}{\alpha^2} \sigma_M^2.
\end{equation}

It remains to translate the convergence result of smoothed gradients into a statement about the true gradient. By adapting Lemma B.1 to our subspace sampling, the smoothing bias is bounded by the subspace dimension $d_{sub}$ rather than the ambient dimension $d$:
\begin{equation} \label{eq:translation}
||\nabla f(X_t)||_F^2 \le 2||\nabla f_\mu(X_t)||_F^2 + \frac{\mu^2 L^2 d_{sub}^2}{2}.
\end{equation}

Combining \eqref{eq:smoothed_sum} and \eqref{eq:translation}, we obtain:
\begin{equation}
\frac{1}{T}\sum_{t=0}^{T-1}\mathbb{E}[||\nabla f(X_t)||_F^2] \le \frac{4\Delta_\mu}{\alpha\eta T} + \frac{2L\eta k}{\alpha} + \frac{2L_m^2}{\alpha^2} \sigma_M^2 + \frac{\mu^2 L^2 d_{sub}^2}{2}.
\end{equation}

Let $f^* = \inf_X f(X)$ and $\Delta_0 = f(X_0) - f^*$. Based on the subspace smoothing bias, $\Delta_\mu \le \Delta_0 + L\mu^2$, by setting the learning rate $\eta = \Theta\left(\sqrt{(\Delta_0 + L\mu^2)/(LkT)}\right)$, we arrive at the final rate:
\begin{equation}
\frac{1}{T}\sum_{t=0}^{T-1}\mathbb{E}[||\nabla f(X_t)||_F^2] \le \mathcal{O}\left( \frac{1}{\alpha}\sqrt{\frac{(\Delta_0+L\mu^2)Lk}{T}} + \frac{L_m^2}{\alpha^2} \sigma_M^2 + \mu^2 L^2 d_{sub}^2 \right).
\end{equation}
This completes the proof, demonstrating that our method efficiently converges to a stationary point with a dependence on the reduced subspace dimension $d_{sub} = rn$ rather than the full parameter dimension. \hfill $\square$
\section{Complexity and Memory Analysis}
\label{app:complexity_analysis}

In Table~\ref{tab:complexity_analysis}, we provide a simple comparison of orthogonalization complexity for ZO-Muon~\citep{lang2026powering} and our method. The key distinction is that ZO-Muon performs Newton-Schulz iterations directly on a subspace matrix $\mathbf{G} \in \mathbb{R}^{r \times n}$, whereas our method applies partial orthogonalization to the subspace momentum $\mathbf{M} \in \mathbb{R}^{r \times n}$ with $k \leq r \ll \min(m,n)$. We only extract top-$k$ dominant spectral directions and eliminate the rest with high noise, which leads to reduced computational cost.

Besides, though we require additional memory to store the power-iteration state $\mathbf{V} \in \mathbb{R}^{n \times k}$ and the subspace momentum $\mathbf{M} \in \mathbb{R}^{r \times n}$, the additional memory cost is negligible compared to the memory storing the full weight matrix (i.e., $mn$) because $r$ and $k$ are much smaller than the matrix dimension. The empirical memory comparison is shown in Table~\ref{tab:efficiency_comparison}. We compare our method with ZO baselines through fine-tuning OPT-13B on SST-2~\citep{wang2019superglue,zhang2022opt}. Our method reaches the highest accuracy of 91.7\%, with almost the same wall-clock time and slightly additional memory usage. These results show that our method achieves a better accuracy performance without the requirement of extra memory and computation overhead.

\begin{table}[t]
    \centering
    \small
    \setlength{\tabcolsep}{10pt}
    \caption{Per-step orthogonalization complexity and optimizer-state memory cost. Here $m$ and $n$ denote the dimensions of the full matrix, and $r$ is the rank of the subspace in ZO-Muon and our method. $k$ is the spectral rank in the SPI, where $k \leq r \ll \min(m,n)$.}
    \label{tab:complexity_analysis}
    \begin{tabular}{lll}
    \toprule
    \textbf{Method} & \textbf{Complexity} & \textbf{Memory Cost} \\
    \midrule
    ZO-Muon & $\mathcal{O}(r^2 n)$ & $0$ \\
    ZO-MOPI & $\mathcal{O}(rnk + nk^2)$ & $rn + nk$ \\
    \bottomrule
\end{tabular}
\end{table}


\begin{table*}[t]
    \centering
    \small
    \setlength{\tabcolsep}{8pt}
    \caption{Comparison of  the efficiency in terms of iteration steps, number of queries, runtime, memory overhead, and accuracy for fine-tuning OPT-13B on SST-2. }
    \label{tab:efficiency_comparison}
    \begin{tabular}{lccccc}
        \toprule
        \textbf{Method} & \textbf{Steps}$\downarrow$ & \textbf{Queries}$\downarrow$ & \textbf{Runtime}$\downarrow$ & \textbf{Memory}$\downarrow$ & \textbf{Accuracy}$\uparrow$ \\
        \midrule
        Adam & 625 & 625 & 10min & 4 $\times$ 63.7 GB & 95.3 \\
        LoRA & 625 & 625 & 6min & 2 $\times$ 52.7 GB & 94.8 \\
        \midrule
        MeZO & 10k & 20k & 2.74h & 25.86GB & 88.3 \\
        LOZO & 10k & 20k & 2.64h & 25.40GB & 87.7 \\
        ZO-Muon & 4k & 20k & 2.05h & 26.22GB & 90.4 \\
        \midrule
        \textbf{ZO-MOPI} & 4k & 20k & 2.03h & 26.56GB & 91.7 \\
        \bottomrule
    \end{tabular}
\end{table*}
\section{Hyperparameter Study}
\label{app:hyper_study}
\begin{figure}[t]
    \centering
    \setlength{\tabcolsep}{1pt}
    \begin{tabular}{cc}
        \includegraphics[width=0.4\linewidth]{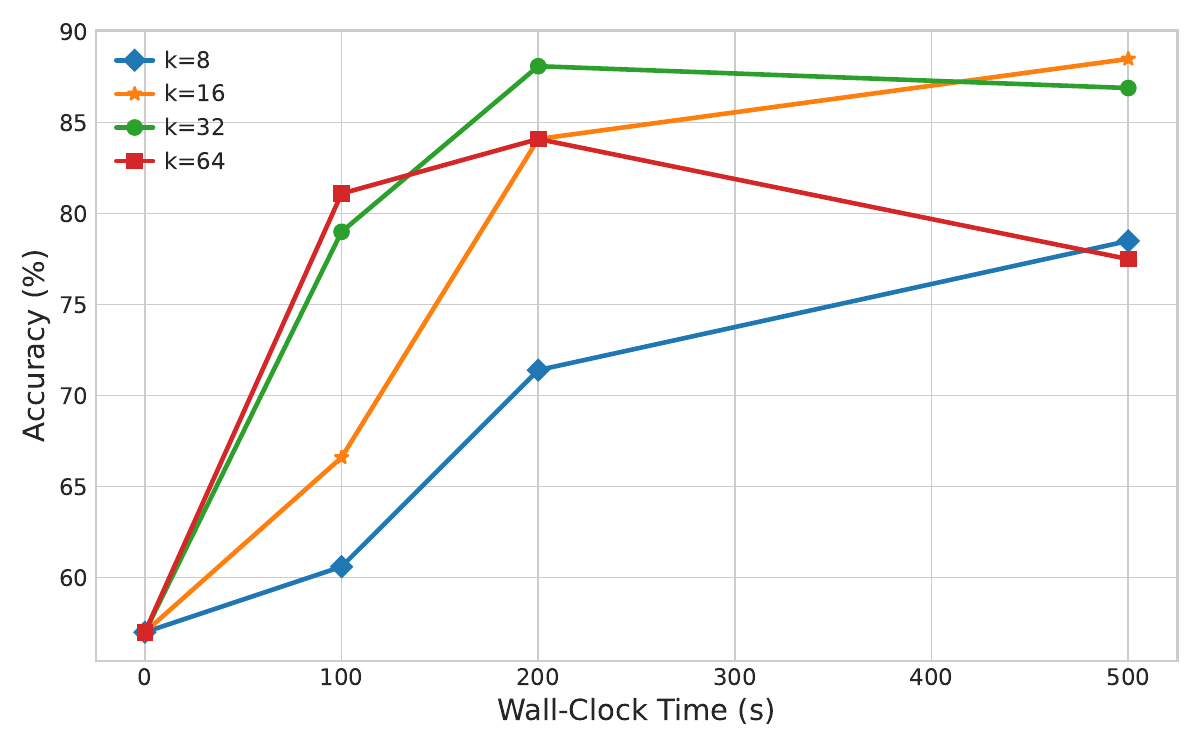} &
        \includegraphics[width=0.4\linewidth]{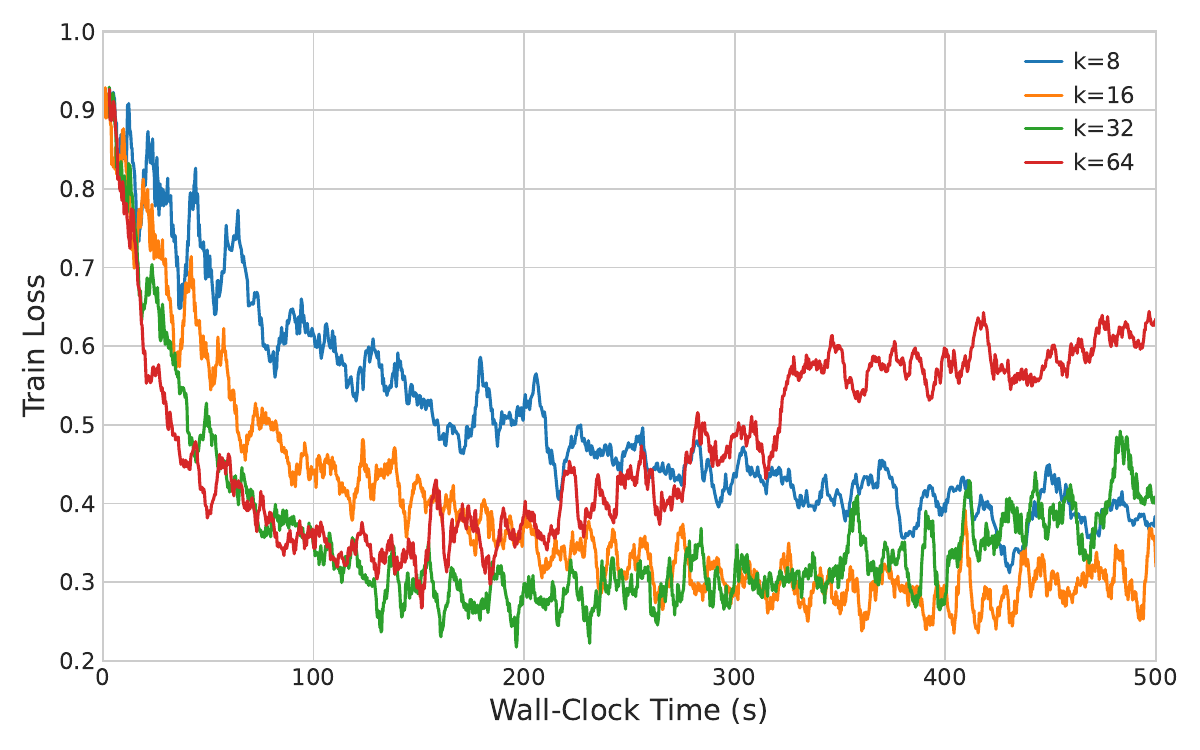} \\
        \small (a) Accuracy vs. wall-clock time &
        \small (b) Training loss vs. wall-clock time
    \end{tabular}
    \caption{Effect of the spectral rank $k$ for fine-tuning OPT-1.3B on SST-2. The two panels show accuracy and training loss under different choices of $k$.}
    \label{fig:hyper_k}
\end{figure}

\begin{figure}[t]
    \centering
    \setlength{\tabcolsep}{1pt}
    \begin{tabular}{cc}
        \includegraphics[width=0.4\linewidth]{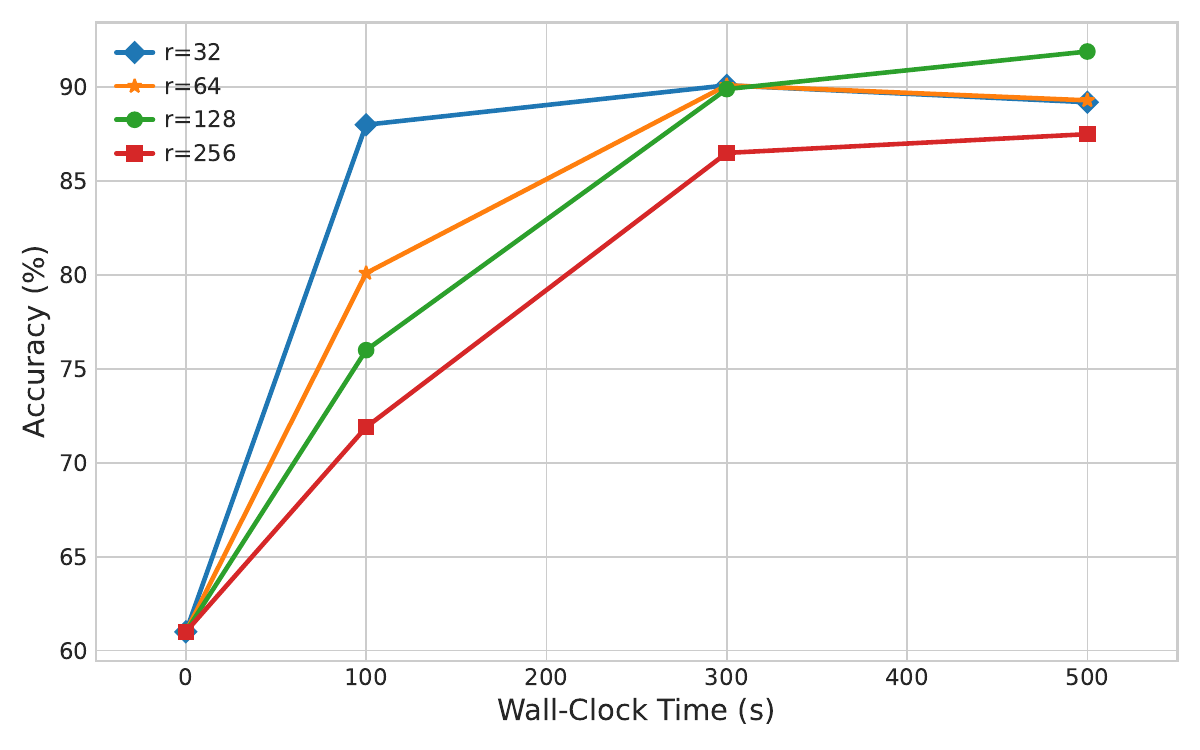} &
        \includegraphics[width=0.4\linewidth]{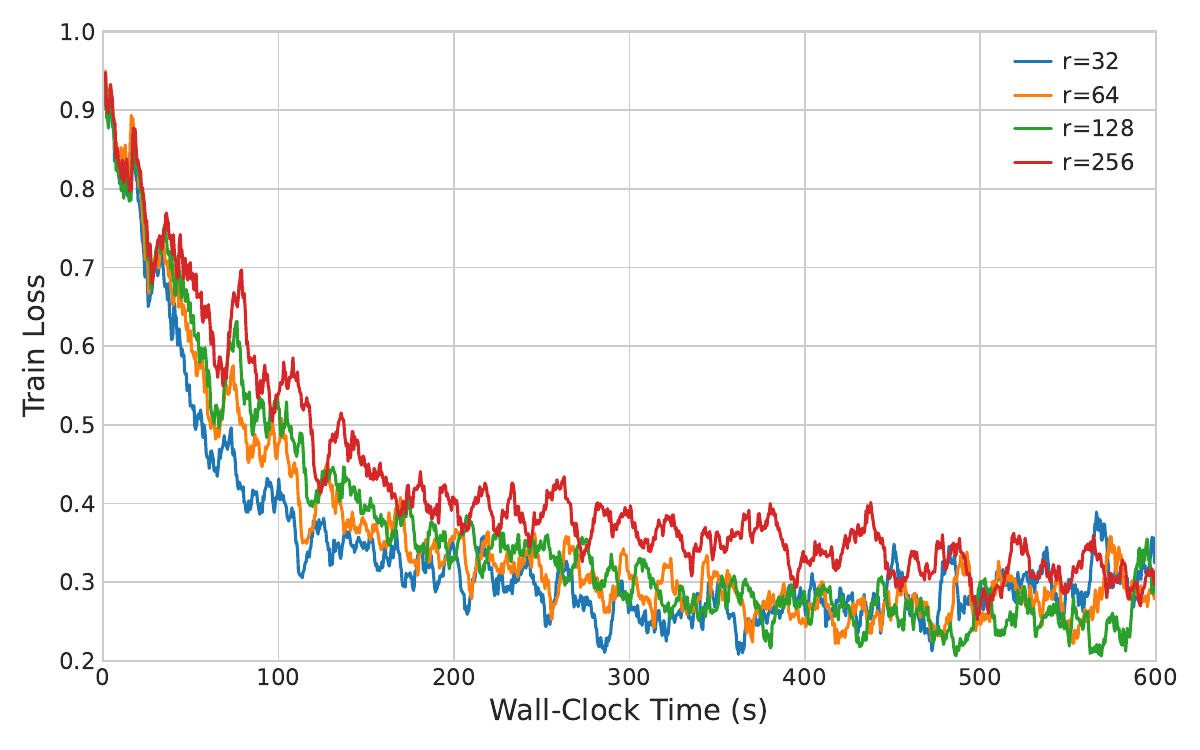} \\
        \small (a) Accuracy vs. wall-clock time &
        \small (b) Training loss vs. wall-clock time
    \end{tabular}
    \caption{Effect of the subspace rank $r$ for fine-tuning OPT-2.7B on SST-2. The two panels show accuracy and training loss under different choices of $r$.}
    \label{fig:hyper_r}
\end{figure}

\paragraph{Different spectral rank $k$}
The spectral rank $k$ controls how many dominant directions we track during the SPI process. We study the effect of different $k$ on fine-tuning OPT-1.3B on SST-2. As shown in Figure~\ref{fig:hyper_k}, when $k$ is too small, the method retains too few dominant directions and cannot sufficiently extract useful spectral information, leading to slower convergence and slower accuracy improvement. When $k$ is too large, SPI includes weaker and noisier spectral components; although this may speed up early training, it reduces the noise-suppression effect and leads to a lower peak accuracy. Thus, an intermediate $k$ gives the best trade-off between signal extraction and noise filtering.


\paragraph{Different subspace rank $r$}
The subspace rank $r$ determines the effective dimension of gradient estimation. We study its effect on fine-tuning OPT-2.7B on SST-2. As shown in Figure~\ref{fig:hyper_r}, both too small and too large $r$ can hurt performance. When $r$ is too small, the estimator has lower variance and can converge quickly at the early stage, but the subspace may miss useful spectral information, leading to a lower final accuracy. When $r$ is too large, the estimator includes more noisy directions, which weakens the variance-reduction effect and makes it harder for SPI to suppress noise. Therefore, a moderate subspace rank provides a better balance between spectral coverage and noise control.

\section{Detailed Experiment Setup}
\label{app:details}
\begin{figure}[t]
    \centering
    \setlength{\tabcolsep}{1pt}
    \begin{tabular}{@{}cc@{}}
        \includegraphics[width=0.4\linewidth]{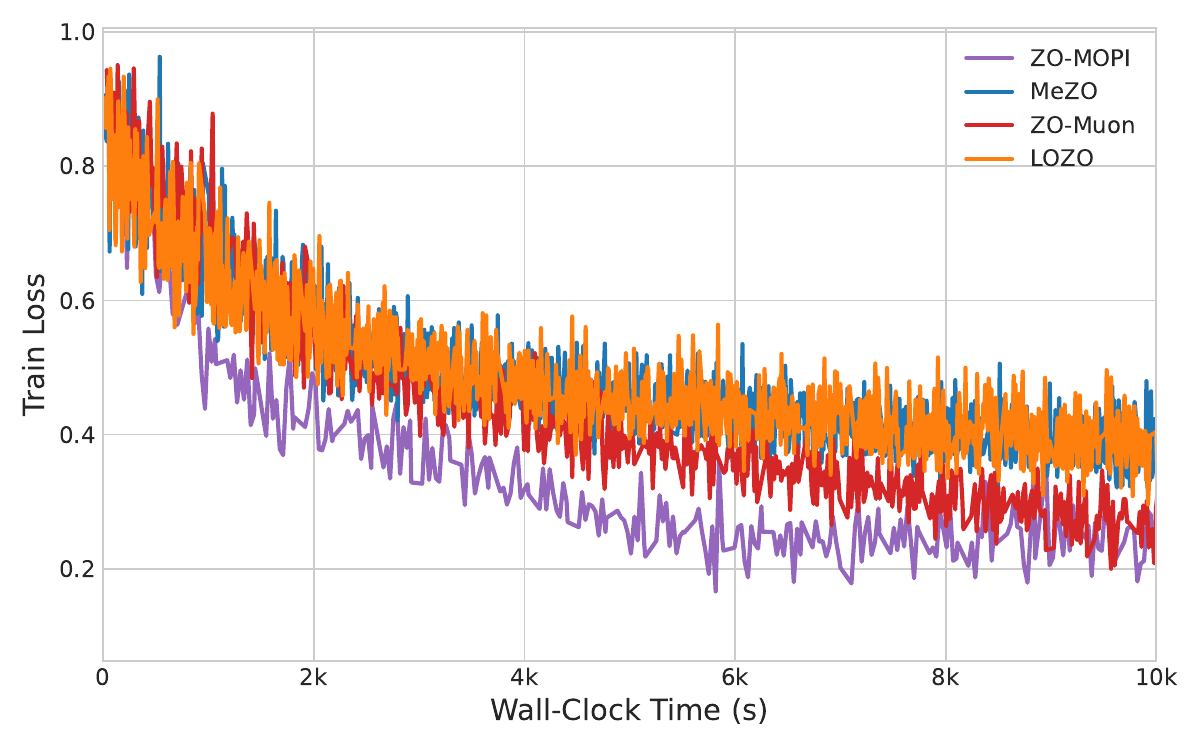} &
        \includegraphics[width=0.4\linewidth]{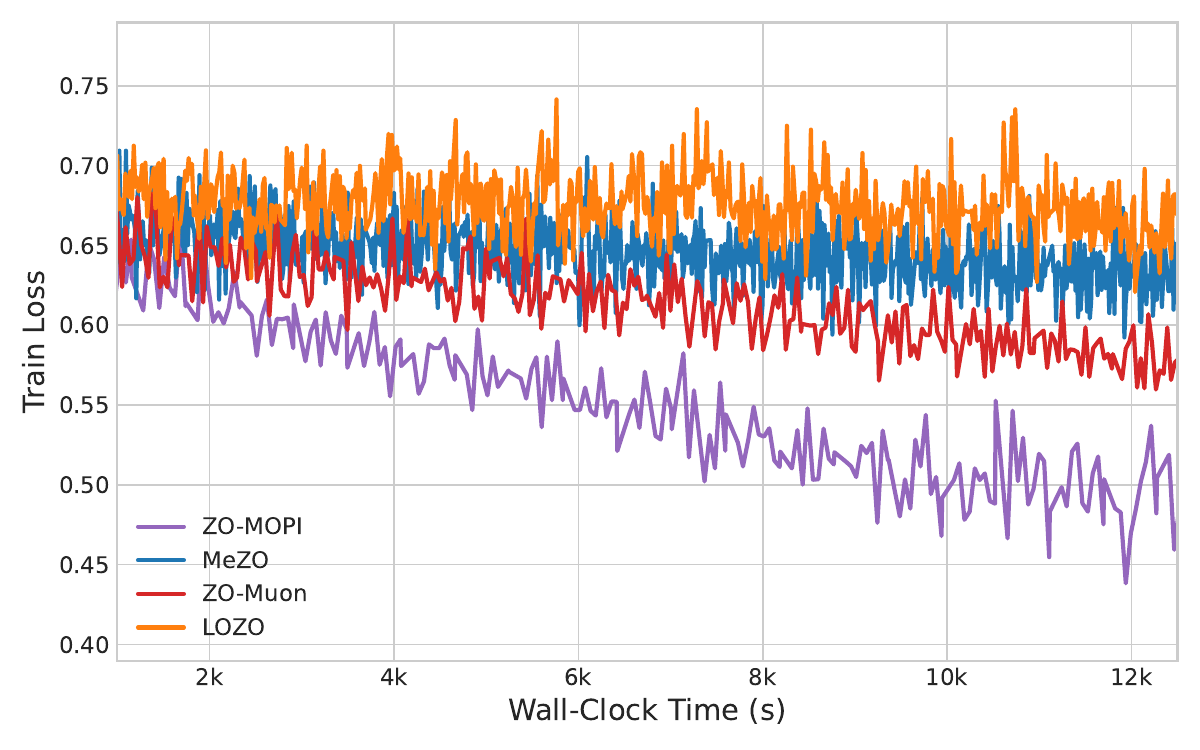} \\
        \small (a) SST-2 & \small (b) RTE \\[6pt]
        \includegraphics[width=0.4\linewidth]{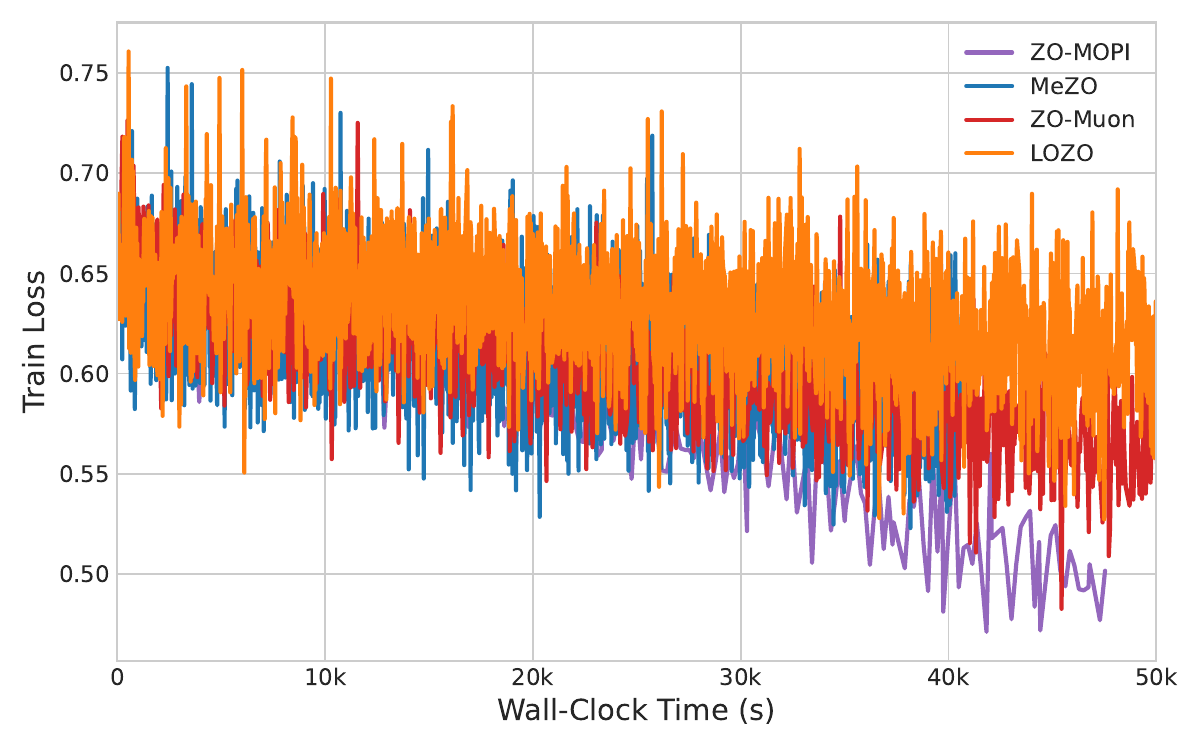} &
        \includegraphics[width=0.4\linewidth]{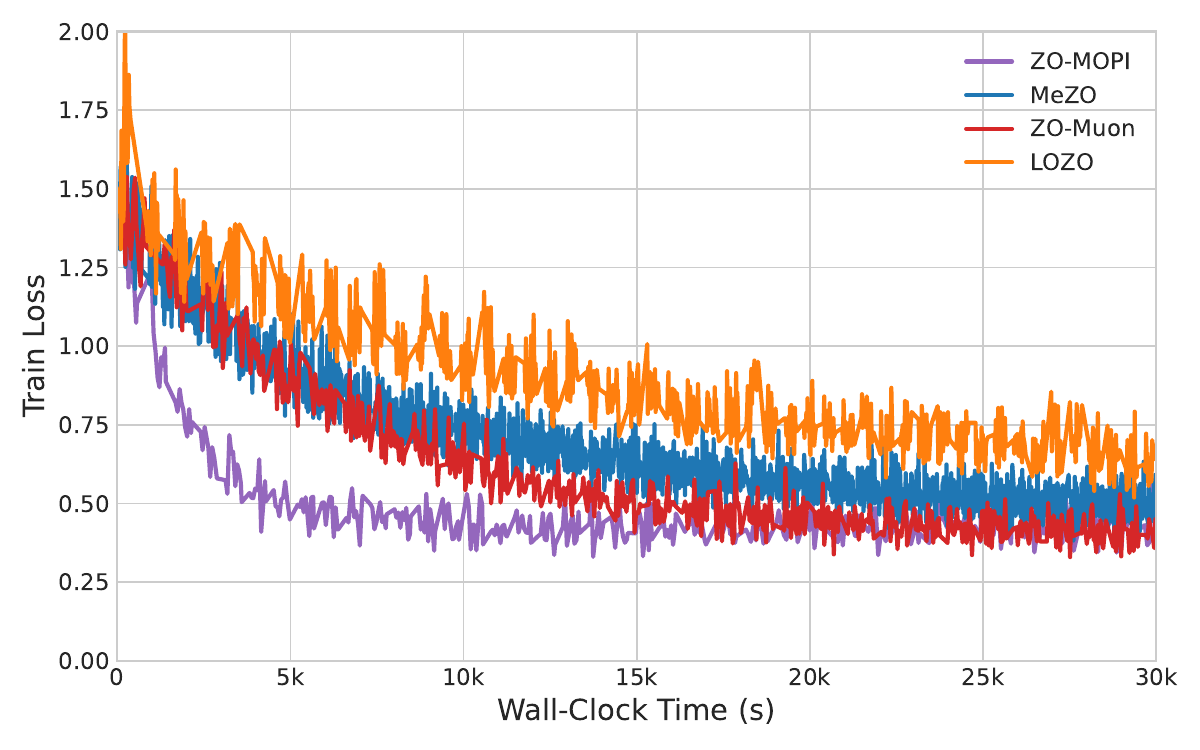} \\
        \small (c) BoolQ & \small (d) SQuAD
    \end{tabular}
    \caption{OPT-13B training loss curves across four SuperGLUE tasks. Each panel reports training loss versus wall-clock time.}
    \label{fig:opt13b_train_loss}
\end{figure}

\begin{figure}[t]
    \centering
    \setlength{\tabcolsep}{1pt}
    \begin{tabular}{@{}cc@{}}
        \includegraphics[width=0.4\linewidth]{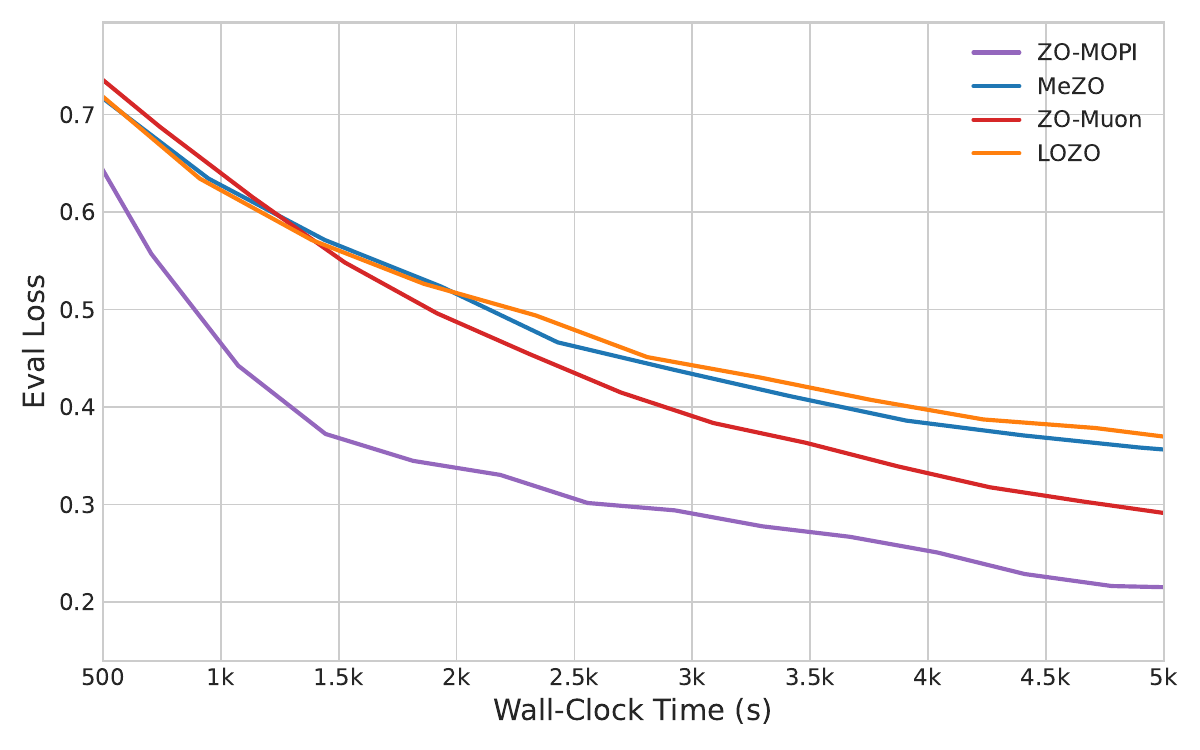} &
        \includegraphics[width=0.4\linewidth]{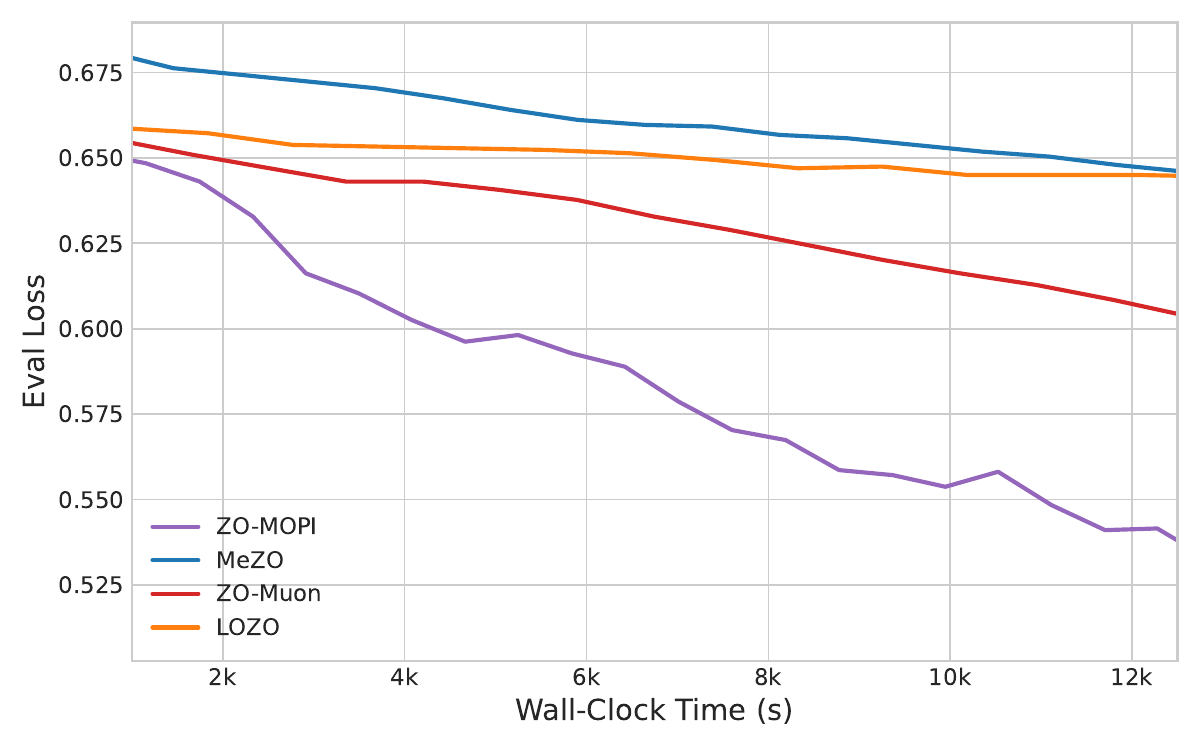} \\
        \small (a) SST-2 & \small (b) RTE \\[6pt]
        \includegraphics[width=0.4\linewidth]{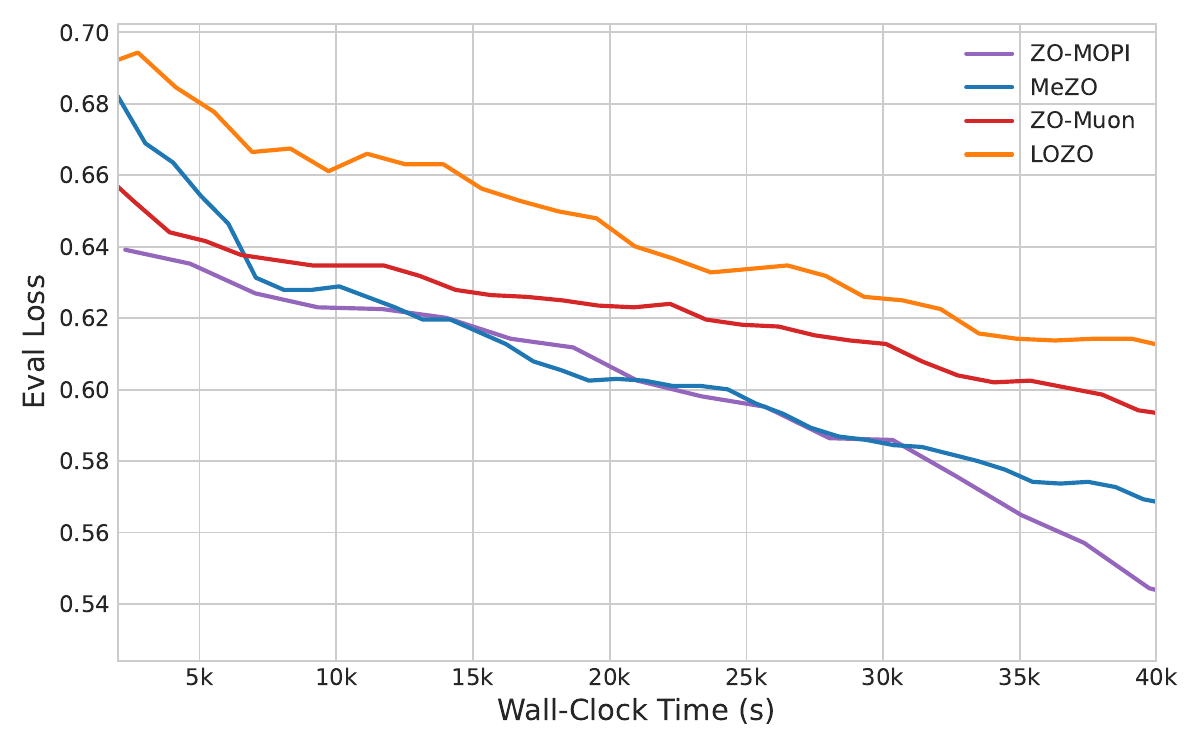} &
        \includegraphics[width=0.4\linewidth]{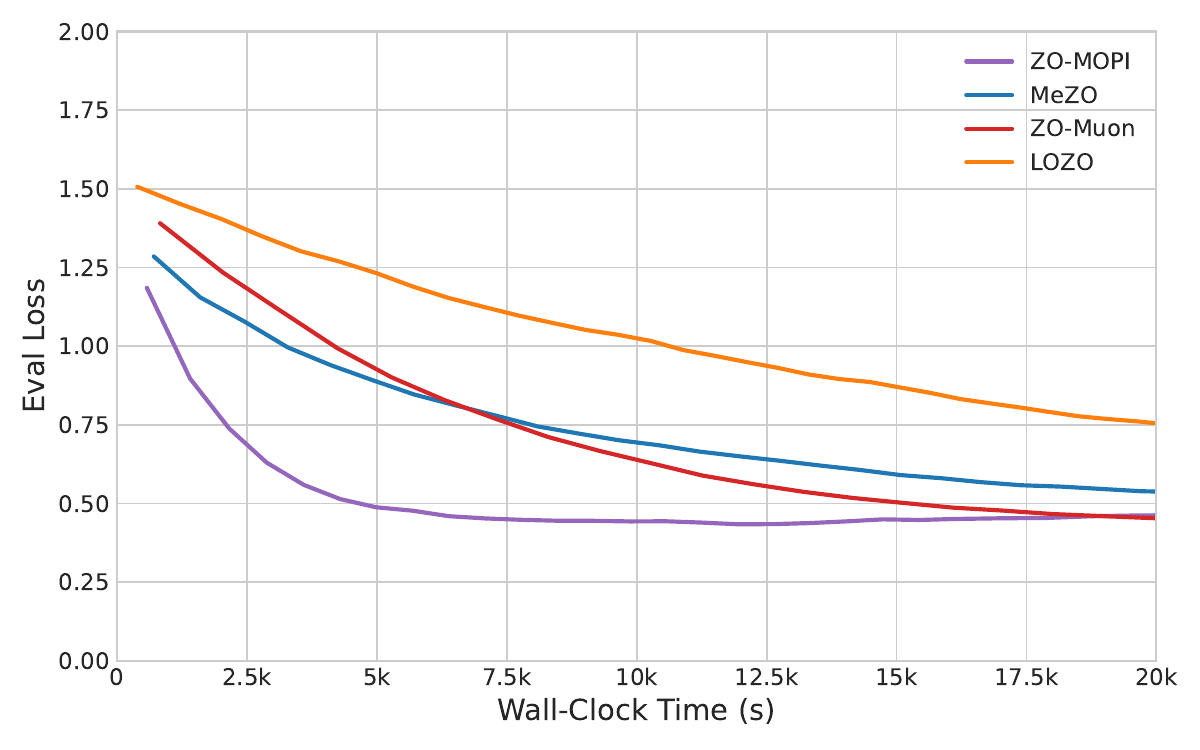} \\
        \small (c) BoolQ & \small (d) SQuAD
    \end{tabular}
    \caption{OPT-13B evaluation loss curves across four SuperGLUE tasks. Each panel reports evaluation loss versus wall-clock time.}
    \label{fig:opt13b_eval_loss}
\end{figure}

\begin{figure}[t]
    \centering
    \setlength{\tabcolsep}{1pt}
    \begin{tabular}{@{}cc@{}}
        \includegraphics[width=0.4\linewidth]{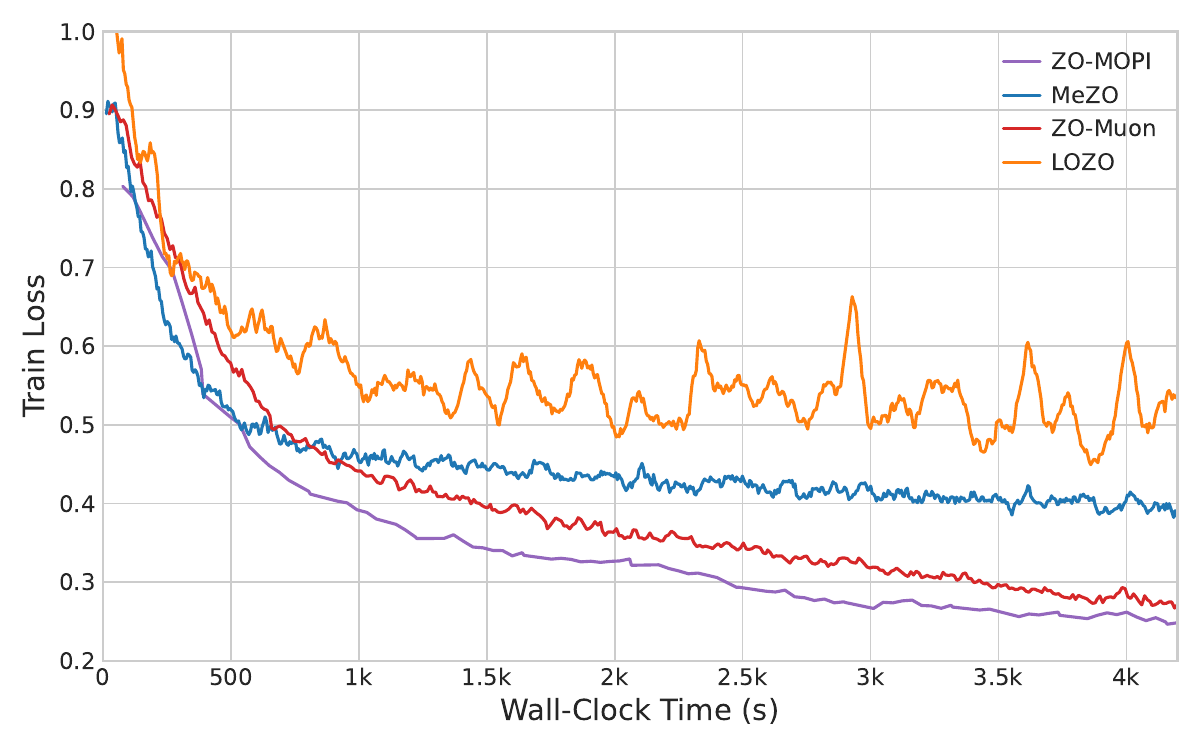} &
        \includegraphics[width=0.4\linewidth]{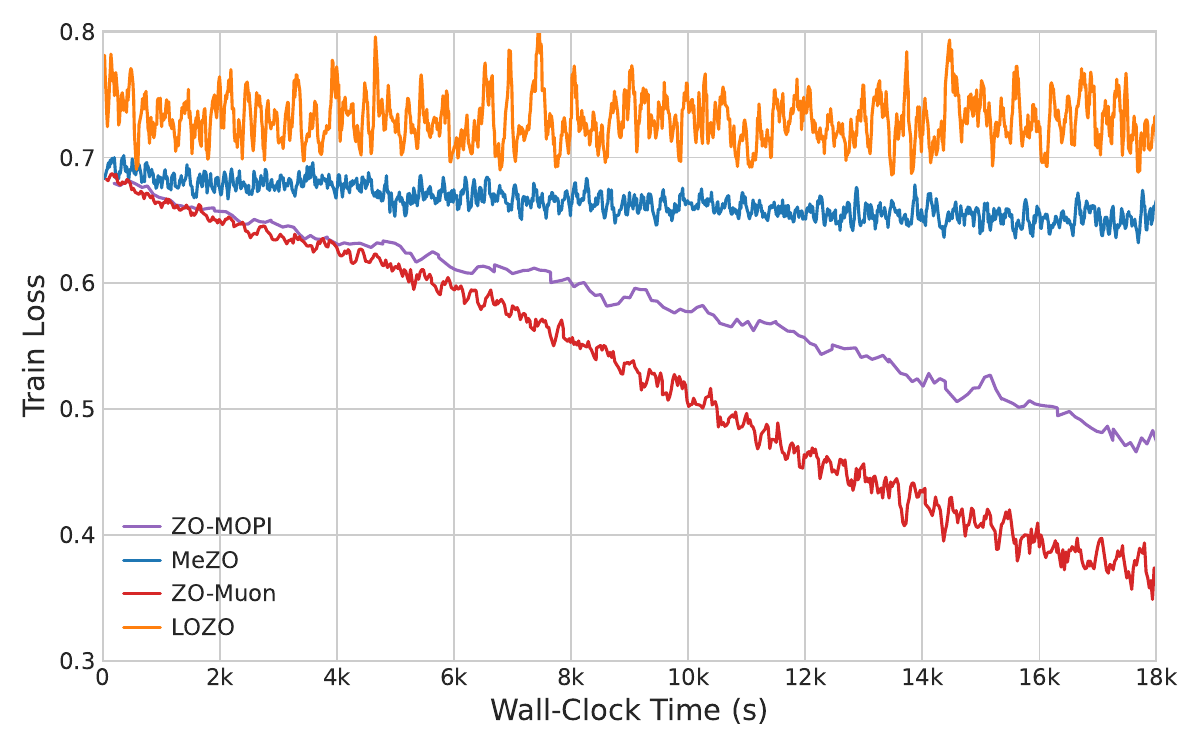} \\
        \small (a) SST-2 & \small (b) RTE \\[6pt]
        \includegraphics[width=0.4\linewidth]{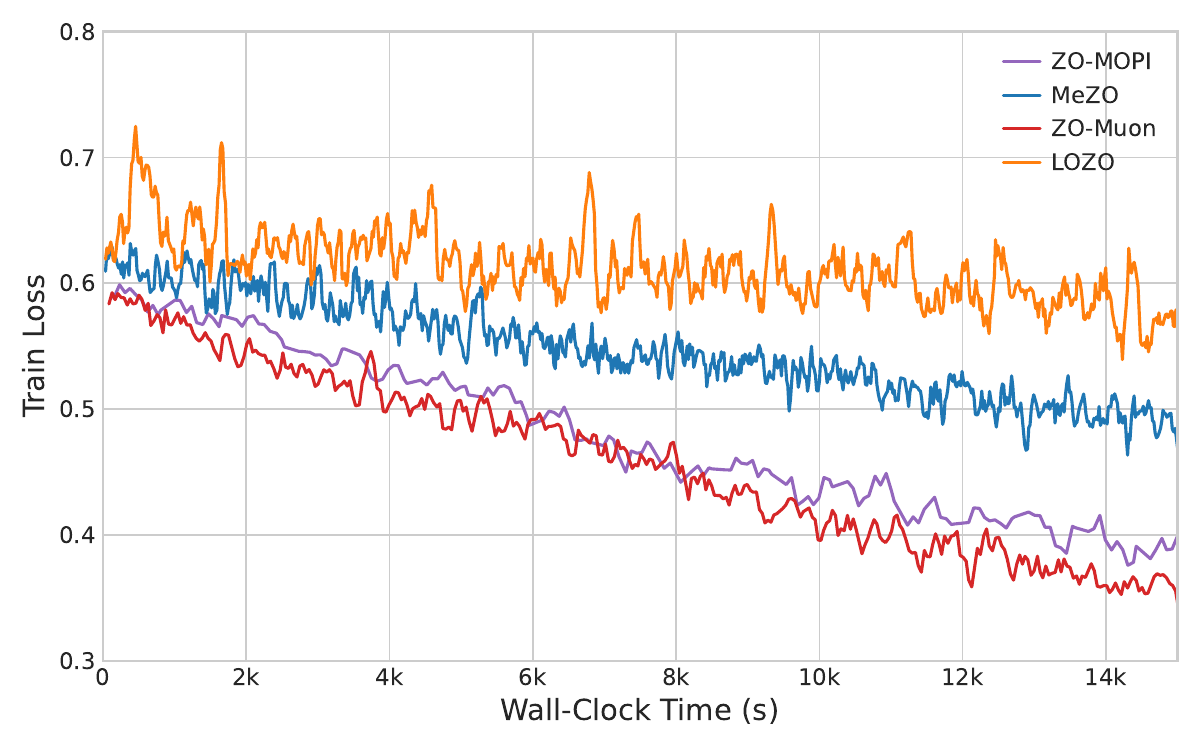} &
        \includegraphics[width=0.4\linewidth]{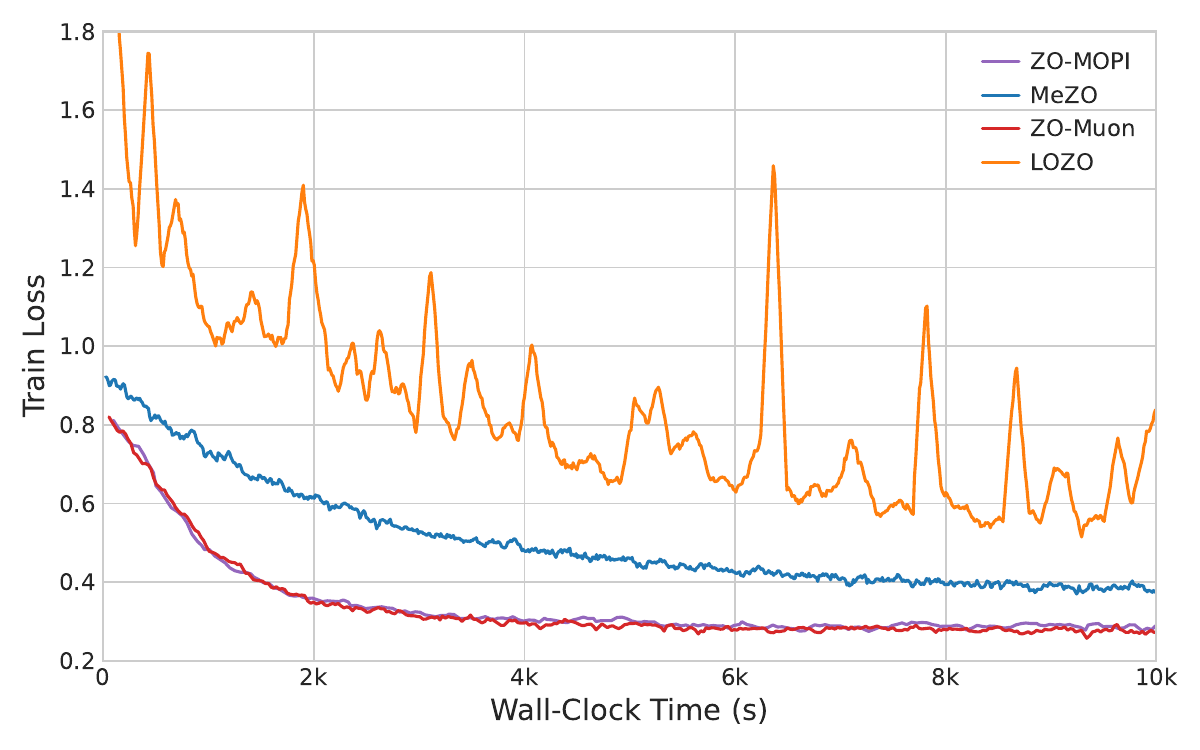} \\
        \small (c) BoolQ & \small (d) SQuAD
    \end{tabular}
    \caption{LLaMA3-8B training loss curves across four SuperGLUE tasks. Each panel reports training loss versus wall-clock time.}
    \label{fig:llama_train_loss}
\end{figure}

\begin{figure}[t]
    \centering
    \setlength{\tabcolsep}{1pt}
    \begin{tabular}{@{}cc@{}}
        \includegraphics[width=0.4\linewidth]{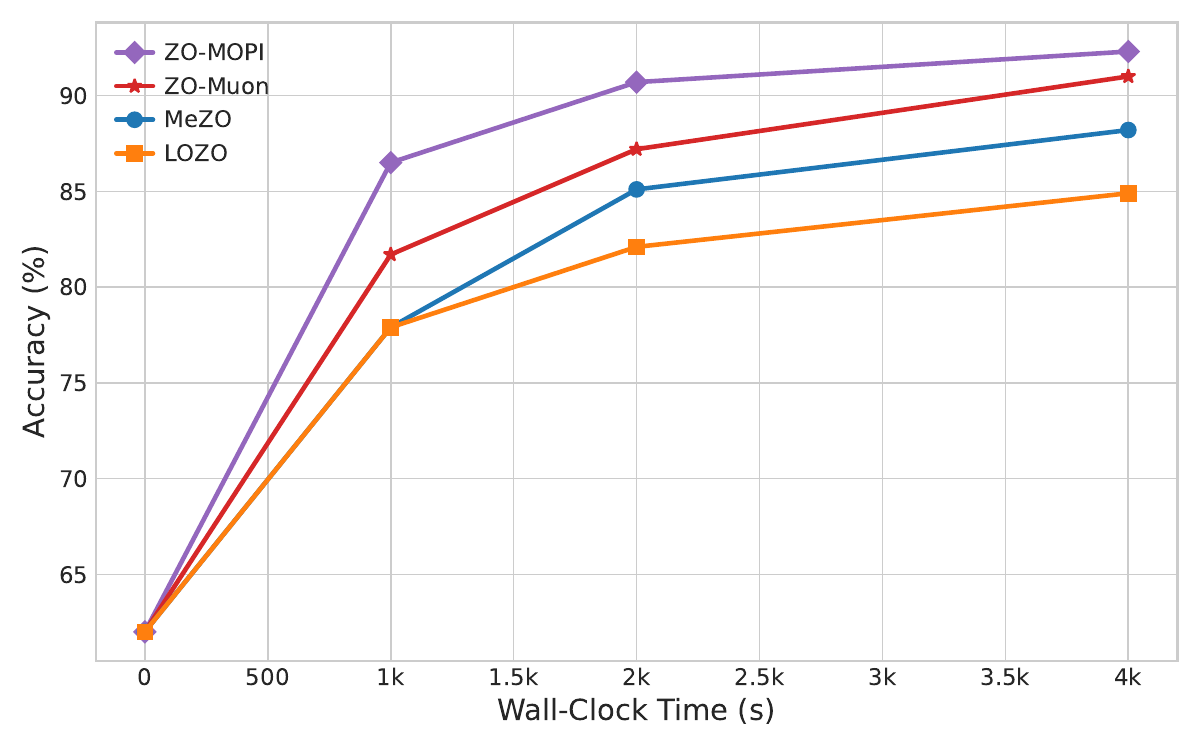} &
        \includegraphics[width=0.4\linewidth]{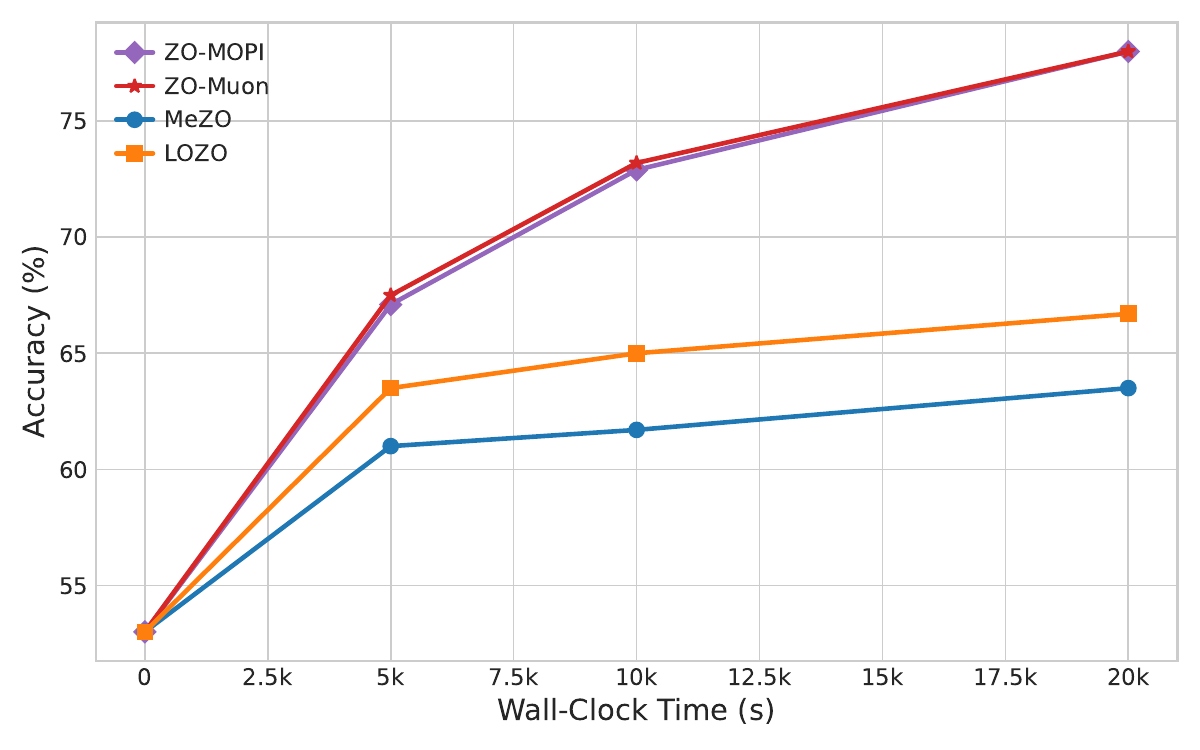} \\
        \small (a) SST-2 & \small (b) RTE \\[6pt]
        \includegraphics[width=0.4\linewidth]{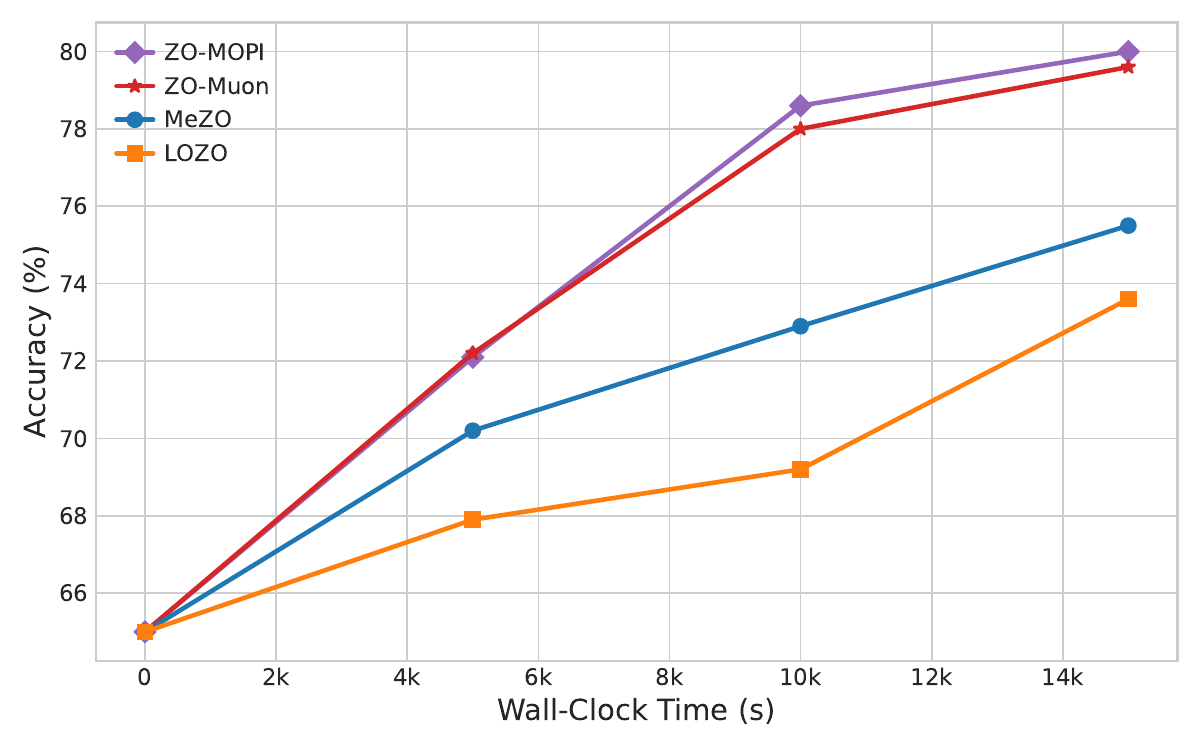} &
        \includegraphics[width=0.4\linewidth]{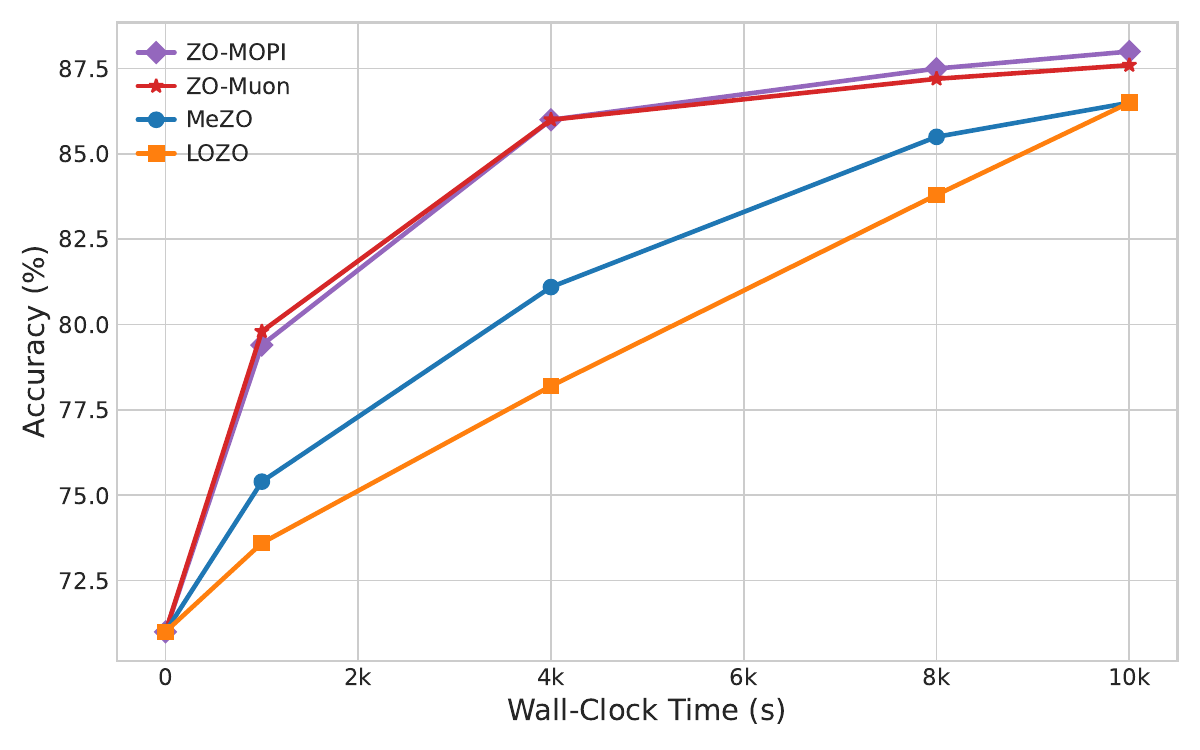} \\
        \small (c) BoolQ & \small (d) SQuAD
    \end{tabular}
    \caption{LLaMA3-8B accuracy curves across four SuperGLUE tasks. Each panel reports accuracy versus wall-clock time.}
    \label{fig:llama3_8b_acc}
\end{figure}

\begin{figure}[t]
    \centering
    \setlength{\tabcolsep}{1pt}
    \begin{tabular}{@{}cc@{}}
        \includegraphics[width=0.4\linewidth]{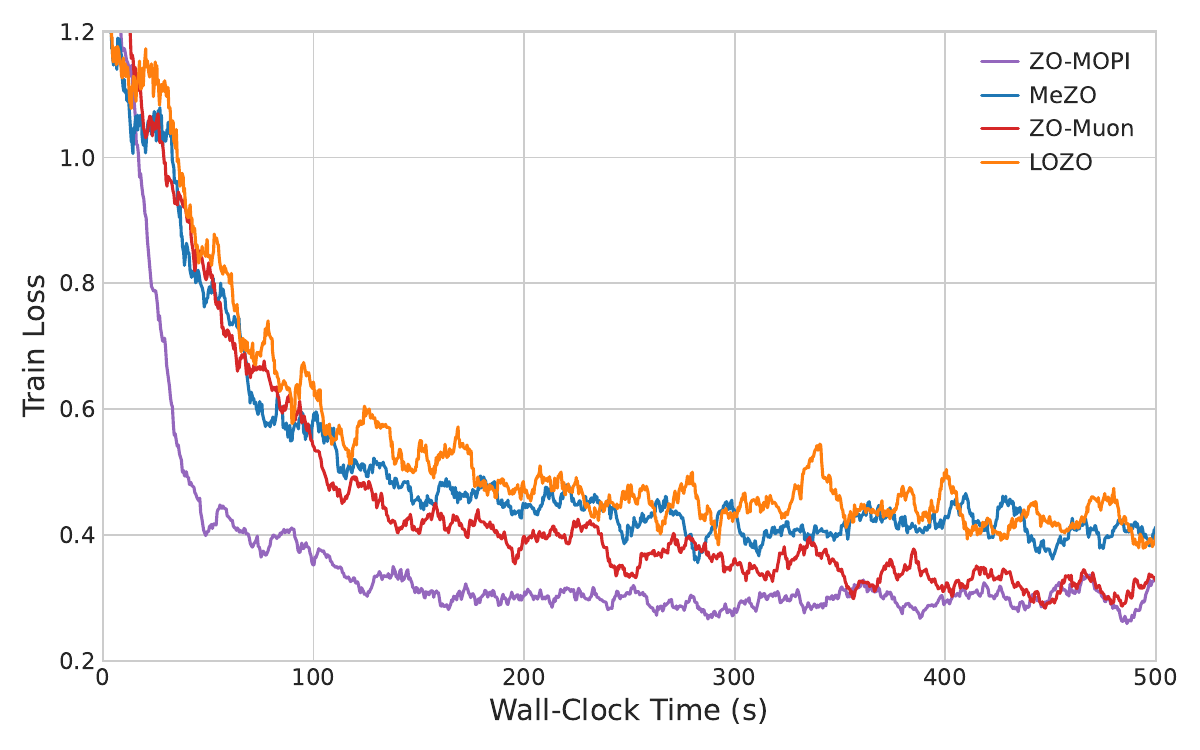} &
        \includegraphics[width=0.4\linewidth]{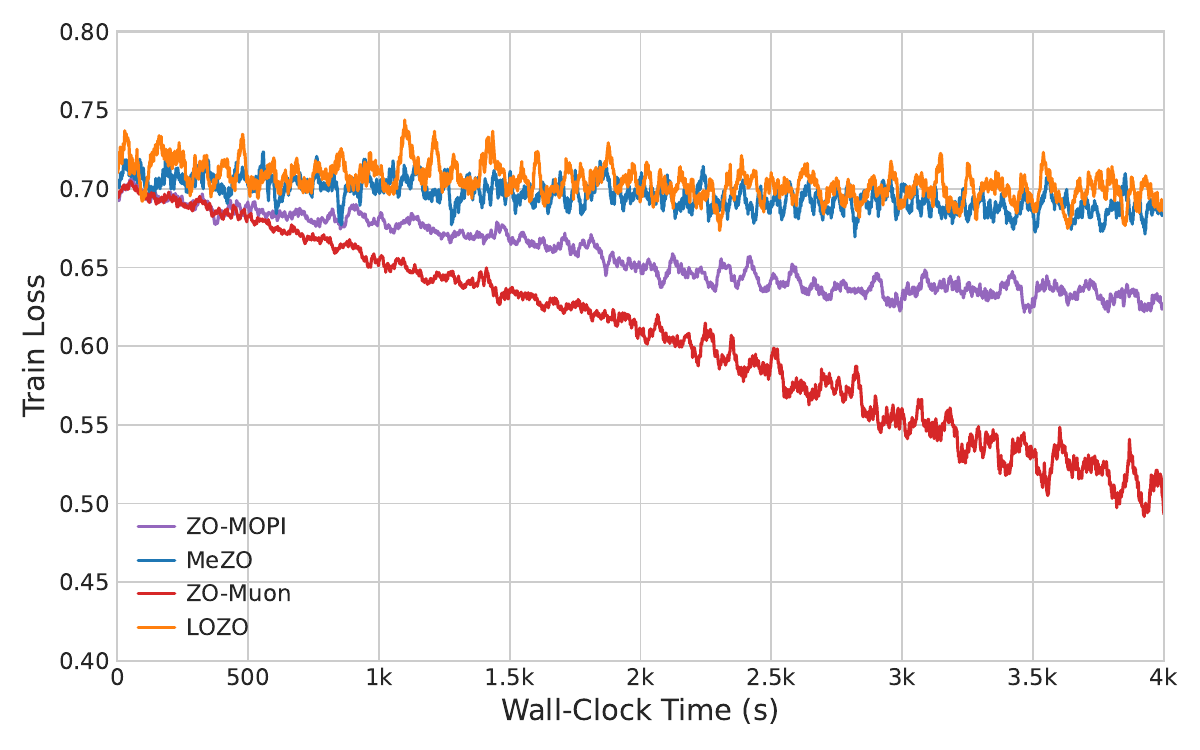} \\
        \small (a) SST-2 & \small (b) RTE \\[6pt]
        \includegraphics[width=0.4\linewidth]{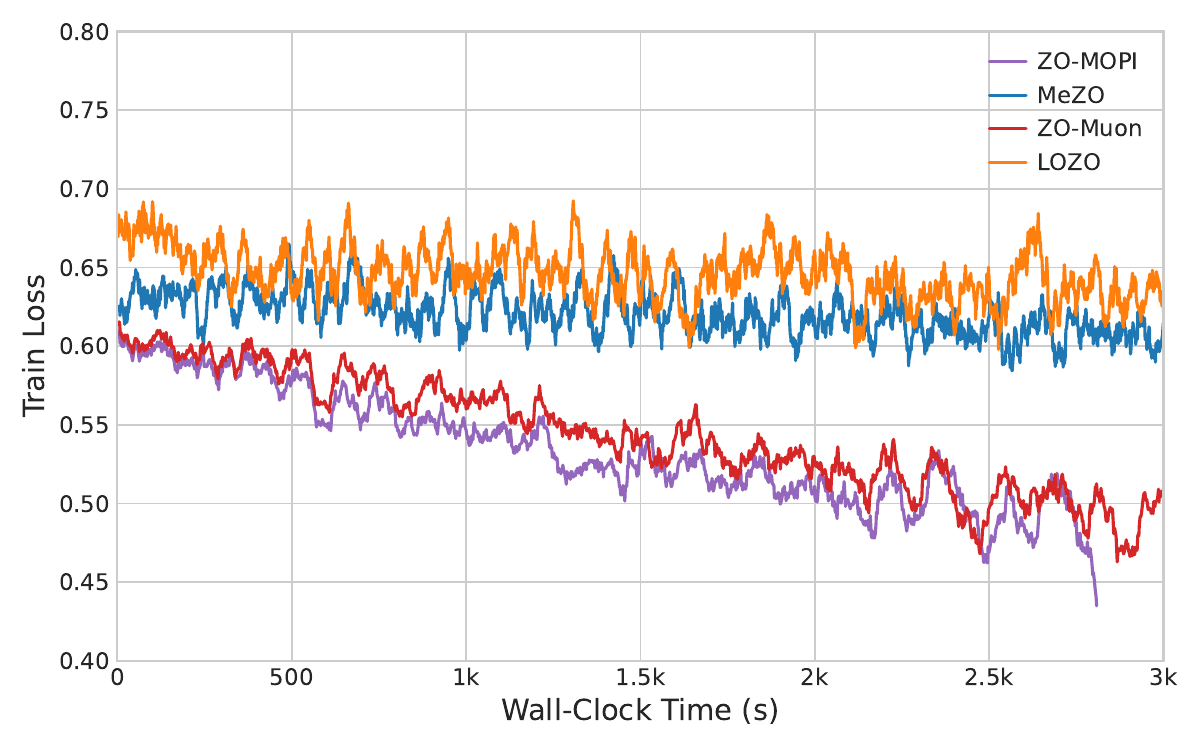} &
        \includegraphics[width=0.4\linewidth]{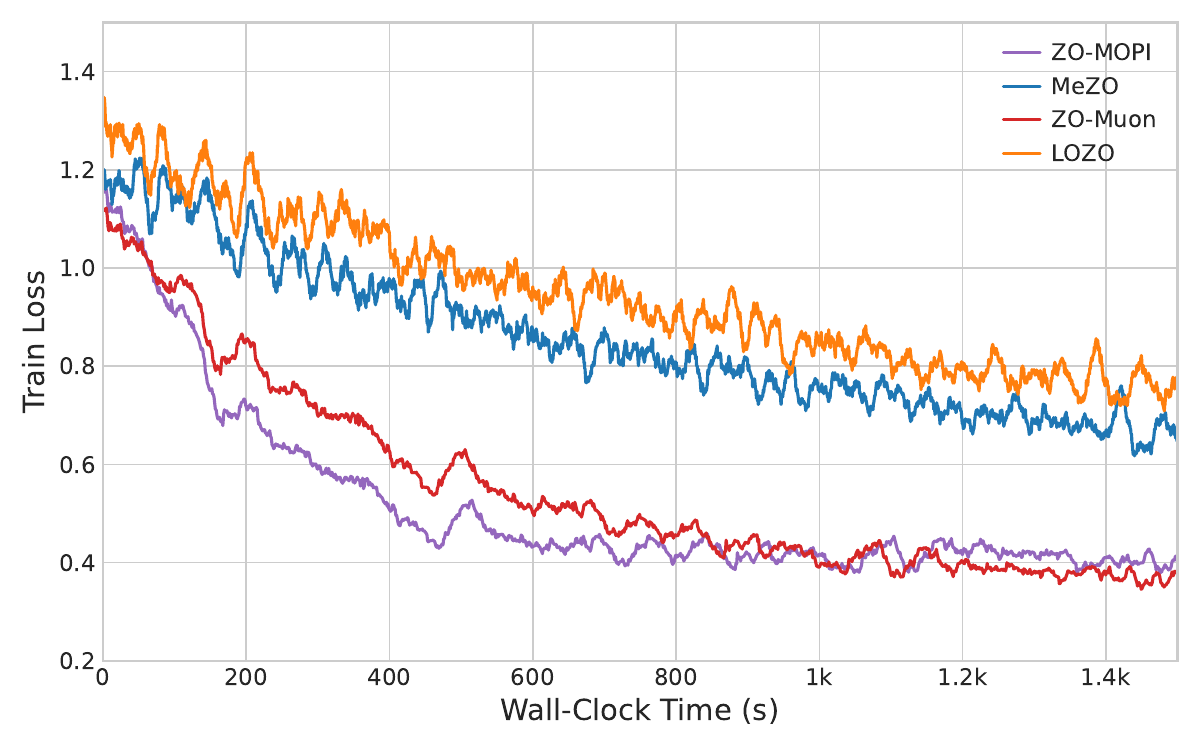} \\
        \small (c) BoolQ & \small (d) SQuAD
    \end{tabular}
    \caption{Gemma2-2B training loss curves across four SuperGLUE tasks. Each panel reports training loss versus wall-clock time.}
    \label{fig:gemma_train_loss}
\end{figure}
\begin{figure}[t]
    \centering
    \setlength{\tabcolsep}{1pt}
    \begin{tabular}{@{}cc@{}}
        \includegraphics[width=0.4\linewidth]{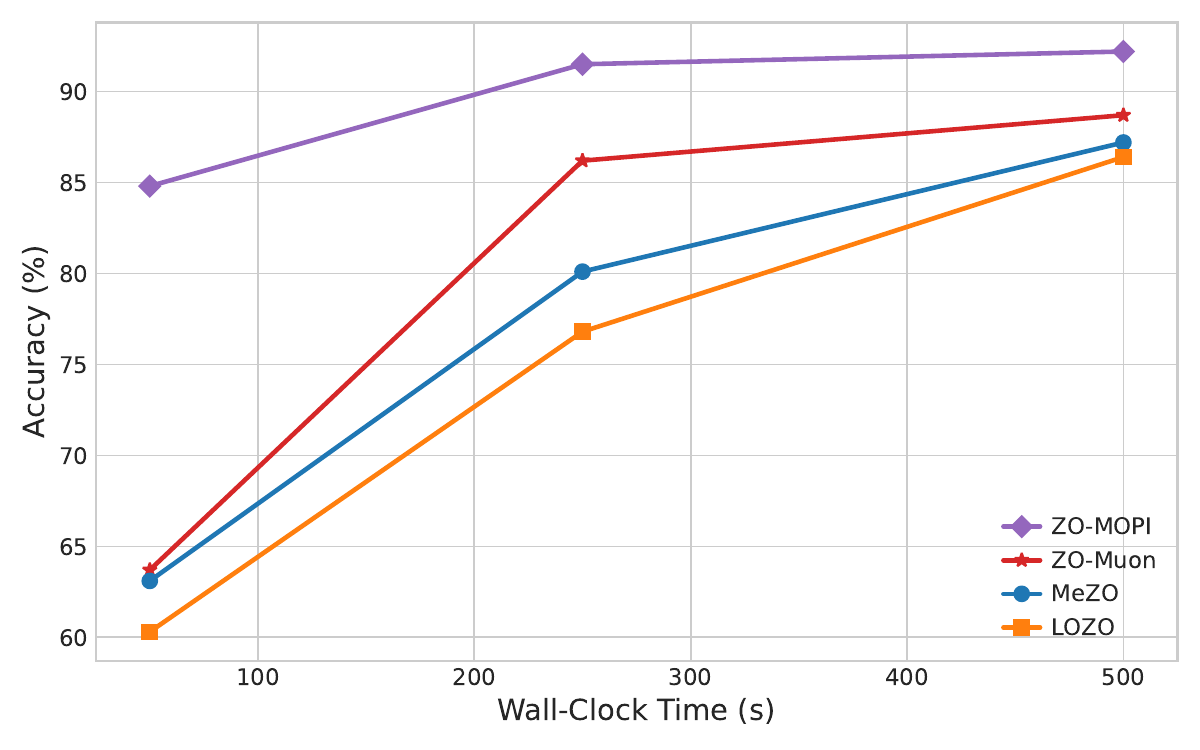} &
        \includegraphics[width=0.4\linewidth]{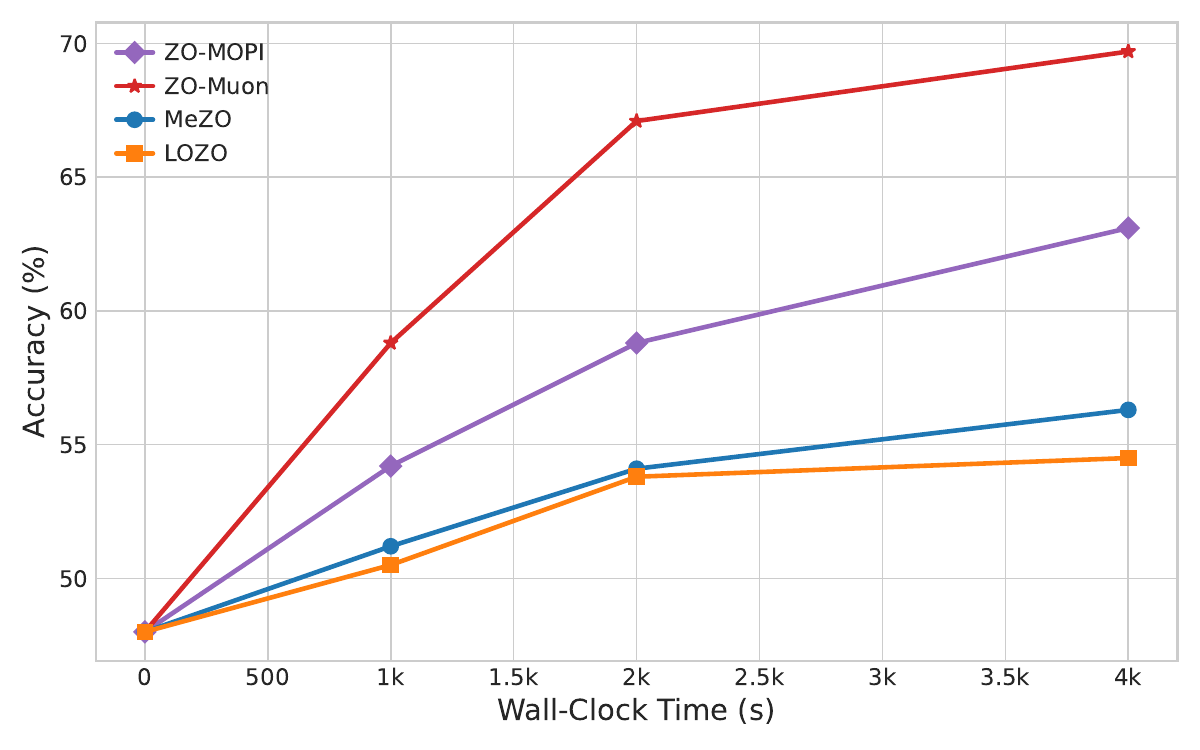} \\
        \small (a) SST-2 & \small (b) RTE \\[6pt]
        \includegraphics[width=0.4\linewidth]{figs/gemma_boolq_acc_vs_time.pdf} &
        \includegraphics[width=0.4\linewidth]{figs/gemma_squad_acc_vs_time.pdf} \\
        \small (c) BoolQ & \small (d) SQuAD
    \end{tabular}
    \caption{Gemma2-2B accuracy curves across four SuperGLUE tasks. Each panel reports accuracy versus wall-clock time.}
    \label{fig:gemma_acc}
\end{figure}

\begin{table*}[t]
    \centering
    \small
    \setlength{\tabcolsep}{8pt}
    \caption{Hyperparameter settings of ZO and FO methods used for fine-tuning LLMs.}
    \label{tab:appendix_hyperparameters}
    \begin{tabular}{lll}
        \toprule
        \textbf{Methods} & \textbf{Hyperparameters} & \textbf{Values} \\
        \midrule
        \multirow{2}{*}{Adam} & Learning rate & $\{5\mathrm{e}{-6}, 1\mathrm{e}{-5}\}$ \\
        & Epochs & 10 \\
        \midrule
        \multirow{4}{*}{LoRA} & Learning rate & $\{5\mathrm{e}{-6}, 1\mathrm{e}{-5}\}$ \\
        & Epochs & 10 \\
        & Rank $(r)$ & 8 \\
        & Alpha $(\alpha)$ & 16 \\
        \midrule
        \multirow{2}{*}{MeZO} & Learning rate & $\{1\mathrm{e}{-7}, 5\mathrm{e}{-7}, 1\mathrm{e}{-6}\}$ \\
        & Train steps & 20k \\
        \midrule
        \multirow{4}{*}{LOZO} & Learning rate & $\{1\mathrm{e}{-7}, 5\mathrm{e}{-7}, 1\mathrm{e}{-6}\}$ \\
        & Rank $(r)$ & $\{2, 4, 8\}$ \\
        & Interval $(\nu)$ & 100 \\
        & Train steps & 20k \\
        \midrule
        \multirow{6}{*}{ZO-Muon} & Learning rate & $1\mathrm{e}{-2}$ \\
        & Rank $(r)$ & $\{64, 128\}$ \\
        & Interval $(\nu)$ & 100 \\
        & Train steps $(N_q = 4)$ & 8k \\
        & Train steps $(N_q = 8)$ & 4k \\
        \midrule
        \multirow{7}{*}{\textbf{ZO-MOPI}} & Learning rate & $\{5\mathrm{e}{-3}, 1\mathrm{e}{-2}\}$ \\
        & Rank $(r)$ & 64 \\
        & Spectral rank $(k)$ & 32 \\
        & Interval $(\nu)$ & 500 \\
        & Train steps $(N_q = 4)$ & 8k \\
        & Train steps $(N_q = 8)$ & 4k \\
        & Train steps $(N_q = 16)$ & 2k \\

        \bottomrule
    \end{tabular}
\end{table*}

\paragraph{Models and Datasets}
For the language model fine-tuning experiments, we consider three representative backbone models: Gemma2-2B~\citep{team2024gemma}, LLaMA3-8B~\citep{grattafiori2024llama}, and OPT-1.3B/13B~\citep{zhang2022opt}. We evaluate them on the SuperGLUE benchmark~\citep{wang2019superglue}, using SST-2~\citep{socher2013recursive}, RTE~\citep{dagan2005pascal},  BoolQ~\citep{clark2019boolq}, and SQuAD~\citep{rajpurkar2016squad}. Following common practice in prior studies on ZO LLM fine-tuning~\citep{malladi2023fine, chen2024enhancing, yu2025zeroth, zhao2024second,lang2026powering}, we construct each run by randomly selecting 1,000 training examples and 1,000 test examples. To keep the comparison aligned with earlier baselines, we use the same prompt templates as MeZO~\citep{malladi2023fine}. All experiments are conducted on a server equipped with four NVIDIA RTX PRO 6000 Blackwell GPUs, each with 96GB memory, under CUDA 13.0. The runtime and memory measurements reported in the efficiency comparison are obtained on the same hardware platform.

\paragraph{Implementation Details}
 Our method, as well as baselines ZO-Muon and LOZO treat model parameters as matrices instead of vectors. Following Muon~\citep{jordan2024muon}, we flatten the last three dimensions of 4D convolutional parameters and treat them as 2D matrices for optimization. Vector parameters, as well as the input and output layers, are optimized with MeZO.

\paragraph{Hyperparameters}
We use hyperparameter settings that are largely aligned with existing literature. In particular, all LLM experiments use a batch size of 16 and a constant learning rate without weight decay. For MeZO and LOZO, we train for 20k steps and fix the perturbation stepsize $\mu=10^{-3}$~\citep{malladi2023fine,chen2024enhancing,lang2026powering}. For our method, we also fix the perturbation stepsize $\mu=10^{-3}$. For algorithmic hyperparameters like $r$ in subspace methods~\citep{chen2024enhancing,lang2026powering}, we strictly follow the original baseline implementations. Because we use multi-point ZO estimation in ZO-Muon and our method, we scale the number of training steps inversely with the number of queries per iteration (e.g. 8k steps for $N_q=4$ and 4k steps for $N_q=8$) to ensure the same total query budget with other baselines. For LoRA~\citep{malladi2023fine}, we follow common settings in prior work and use rank $r=8$ with scaling factor $\alpha=16$. Detailed settings for both ZO and FO methods are summarized in Table~\ref{tab:appendix_hyperparameters}.

\paragraph{Additional Experiment Results}
Figures~\ref{fig:opt13b_train_loss} through \ref{fig:gemma_train_loss} illustrate the training dynamics of OPT-13B, LLaMA3-8B, and Gemma2-2B on SuperGLUE~\citep{wang2019superglue,zhang2022opt,team2024gemma,grattafiori2024llama}. Notably, the accuracy-vs-time curves demonstrate a much steeper convergence trend than existing ZO baselines across most tasks, which closely aligns with the wall-clock speedups in Figure~\ref{fig:opt13b_acc_vs_time} and the quantitative results in Section~\ref{sec:experiments}. These results underscore our method's potential to significantly improve fine-tuning efficiency in resource-constrained environments.

\end{document}